\documentclass{article}

\usepackage{style}

\usepackage{times}
\usepackage{soul}
\usepackage{url}
\usepackage[hidelinks]{hyperref}
\usepackage[utf8]{inputenc}
\usepackage[small]{caption}
\usepackage{graphicx}
\usepackage{amsmath}
\usepackage{amsthm}
\usepackage{booktabs}
\usepackage{algorithm}
\usepackage{algorithmic}
\usepackage[switch]{lineno}
\usepackage{float}

\usepackage{microtype}
\usepackage{graphicx}
\usepackage{subcaption}
\usepackage{booktabs}

\usepackage{hyperref}

\usepackage{mathtools}
\usepackage{algorithm}
\usepackage{algorithmic}
\usepackage{xcolor}
\usepackage{booktabs}
\usepackage{multicol}
\usepackage{enumitem}
\usepackage{tikz}
\usepackage{amsmath}
\usepackage{amsthm}
\usepackage{amssymb} 
\usepackage{array}
\usepackage{multirow}
\usepackage{pgfplots}
\pgfplotsset{compat=1.18}
\usepackage{adjustbox}
\usetikzlibrary{arrows.meta, shapes.symbols, calc}

\usepackage{newfloat}
\usepackage{listings}
\DeclareCaptionStyle{ruled}{labelfont=normalfont,labelsep=colon,strut=off} 
\floatstyle{ruled}
\newfloat{listing}{tb}{lst}{}
\floatname{listing}{Listing}
\lstdefinestyle{python}{
    language=Python,
    basicstyle=\ttfamily\footnotesize,
    keywordstyle=\color{blue}\bfseries,
    stringstyle=\color{red},
    commentstyle=\color{green!50!black},
    numberstyle=\tiny,
    stepnumber=1,
    numbersep=5pt,
    frame=lines,
    breaklines=true,
    showstringspaces=false,
    captionpos=b,
    tabsize=4,
    escapeinside={(*@}{@*)}
}

\usepackage{scalerel}

\newcommand{\prob}[0]{\mathbb{P}}
\newcommand{\inputspace}[0]{\mathcal{I}}
\newcommand{\outputspace}[0]{\mathcal{O}}

\newcommand{\Sfunction}[0]{S}

\newcommand{\unknown}[0]{\star}

\newcommand{\setofvaluations}{\mathcal{V}}

\newcommand{\low}{\mathit{low}}
\newcommand{\up}{\mathit{up}}

\newcommand{\stopfunction}{\texttt{STOP}}
\newcommand{\searchheuristic}{\texttt{H}}

\newcommand{\tot}[1]{\mathrm{tot}({#1})}

\makeatletter
\@ifundefined{citep}{\DeclareRobustCommand{\citep}[1]{\cite{#1}}}{}
\@ifundefined{citet}{\DeclareRobustCommand{\citet}[1]{\citeauthor{#1}~\shortcite{#1}}}{}
\makeatother

\newtheorem{theorem}{Theorem}

\title{DPNL: A DPLL-based Algorithm for Probabilistic Neurosymbolic Learning}

\author{
Thomas Valentin\and
Pierre Genevès\and
Luisa Werner\and
Sarah Chlyah\and
Nabil Layaïda\and
Nils Gesbert\\
\affiliations
Tyrex team, Univ. Grenoble Alpes, CNRS, Inria, Grenoble INP, LIG
Grenoble, France, 38000
}

\begin{document}
 
\maketitle

\begin{abstract}
Probabilistic Neurosymbolic Learning (PNL) combines neural predictions
with symbolic reasoning, enabling end-to-end learning from final-output
supervision without labels for intermediate concepts.
A central challenge is probabilistic inference: state-of-the-art
approaches often rely on materializing the logical provenance of a
query, which can itself become a major computational bottleneck.
We introduce Dynamic Probabilistic Neurosymbolic Learning (DPNL), an oracle-guided framework that avoids requiring complete provenance materialization before inference.
DPNL lazily explores the space of intermediate assignments, while oracles resolve entire regions that can already be certified to produce or exclude the target output.
We establish conditions ensuring soundness and termination.
ApproxDPNL extends the same search with early termination while
maintaining certified bounds on the exact output probability, providing
controlled approximation guarantees.
The oracle interface decouples inference from the representation of the
symbolic component, enabling problem-specific reasoning within the same
framework.
Experiments on several neurosymbolic tasks show that DPNL and
ApproxDPNL substantially extend the range of problem instances tractable
by probabilistic neurosymbolic inference.
\end{abstract}

\section{Introduction}
Neurosymbolic artificial intelligence seeks to combine the learning
capabilities of neural networks with structured symbolic reasoning
\cite{NeSy:Survey:vanHarmelen2019,NeSy:Survey:Varmeulen2023}.
Probabilistic Neurosymbolic Learning (PNL) is one approach to this
integration, in which neural networks predict probability distributions
over intermediate concepts and a symbolic component combines these
predictions to produce the final output
\cite{NeSy:Survey:DeRaedt2020,NeSy:Scallop:Siyang2023}.
A key advantage of this setting is that the intermediate predictions
need not be supervised directly. Instead, the symbolic computation
connects them to the observed output, allowing the neural component to
be trained end-to-end from supervision on the final task.
At the same time, the symbolic component makes explicit how intermediate predictions are combined to determine the final output, providing a structured and inspectable reasoning process.

Consider, for instance, a kinship-reasoning task in the style of
CLUTRR~\cite{NeSy:CLUTRR:Sinha2019}.
Given a short text describing relations between people, a neural
component may predict distributions over uncertain relations such as
\emph{sister}, \emph{aunt}, or \emph{parent}.
Symbolic kinship rules then compose these relations to answer a query
such as ``How are Rich and Anne related?''.
Training may provide only the correct answer to this final query, and
not the intermediate relations required to derive it.
Learning therefore requires propagating supervision through the
symbolic reasoning process while accounting for uncertainty in the
intermediate neural predictions.

This introduces a central computational challenge.
To obtain the probability of a given output, probabilistic reasoning
must aggregate the probability mass of all assignments of the
intermediate predictions that make the symbolic computation produce
that output.
Depending on the task, this aggregation may admit an efficient
computation or a compact representation.
In general, however, the number of relevant assignments can grow
rapidly with the number of intermediate variables, making their explicit
enumeration impractical.
Importantly, this difficulty is distinct from the cost of evaluating the
symbolic computation on one complete assignment: a symbolic function
may be inexpensive to evaluate while the set of assignments leading to
a given output is much harder to represent or aggregate over.

Existing exact PNL approaches such as
DeepProbLog~\cite{NeSy:DeepProbLog:mostrecentjournalversion2021} address this problem
by constructing the logical provenance of a query, using knowledge compilation techniques, and subsequently
performing probabilistic inference over that representation.
However, constructing the complete provenance can itself become a bottleneck when
the number of alternative reasoning paths is large.
Approximate methods have consequently been proposed to extend PNL to
larger problem instances
\cite{NeSy:DeepProbLog:Manhaeve2021,NeSy:Scallop:Huang2021,NeSy:A-NeSI:vanKrieken2023}.
These methods trade exact inference for scalability in different ways.
This motivates a complementary question: can probabilistic inference
avoid constructing the complete provenance in the first place, while
retaining exact inference when feasible and providing explicit control
over the error when approximation is necessary?

\paragraph{Contribution.}
Our main contribution is the DPNL framework, a new probabilistic
reasoning architecture for end-to-end neurosymbolic learning, with two major aspects:

\begin{enumerate}[label=\roman*.]
    \item \textbf{Inference without full provenance
    materialization.}
Instead of constructing the complete logical provenance of a query before inference, DPNL performs a lazy exploration of partial assignments and dynamically eliminates regions of the implicit reasoning space that can be determined to succeed or fail. Only regions that cannot yet be decided are explored
further. This yields an exact inference procedure when exploration is continued to completion.

    \item \textbf{Certified approximation.}
When exact inference is not required, ApproxDPNL can stop this exploration early while maintaining certified bounds on the exact output probability. Specifically, it tracks the probability mass of unresolved regions and maintains certified lower and upper bounds on the exact output probability at any point during the search.
 
\end{enumerate}

An important architectural property of DPNL and ApproxDPNL is their ability to combine neural predictions over finite-domain random variables with diverse forms of symbolic reasoning. 
They separate probabilistic inference from the representation and
implementation of the symbolic component, allowing the same inference architecture to accommodate a broad range of symbolic representations and reasoning procedures. 
We formalize a generic oracle interface
that supports ordinary
executable symbolic programs, solver-backed logical specifications,
and domain-specific oracles, and establish conditions under which
oracle-guided inference is sound and terminating.

We implement DPNL and ApproxDPNL in a research prototype and evaluate them on several neurosymbolic tasks.
Experiments show that oracle-guided reasoning can substantially extend the range of tractable problem instances for probabilistic neurosymbolic learning.

\section{Problem Setting and Definitions}
\label{sec:problem}

A probabilistic neurosymbolic learning (PNL) problem combines a neural
component with a symbolic reasoning component, with the goal of learning
intermediate neural predictions from supervision on the final symbolic
output.

Given a fixed observed input \(i\) from an input space \(\mathcal I\), the neural
component predicts a distribution for each of \(m\) random
variables \(X_1(i), \ldots, X_m(i)\), where each \(X_k(i)\) takes values in
a finite domain \(V_k\). 
We write
\[
    \pi_{\theta,k}(x \mid i)
    = \prob\bigl(X_k(i) = x\bigr)
\]

for the predicted distribution of $X_k(i)$, where $\theta$ collects
all neural parameters. 
The neural component may consist of one or several neural networks.

The symbolic component is a deterministic function
\[
    \Sfunction : V_1 \times \cdots \times V_m
    \rightarrow \outputspace,
\]
which maps a joint assignment of the neural variables to an output
$o \in \outputspace$.

For instance, in the kinship example, the variables \(X_k(i)\) represent
uncertain relations extracted from the text, while \(S\) applies symbolic reasoning to these relations to infer the answer to the final kinship query.

For each fixed input \(i\), the random variables
\(X_1(i),\ldots,X_m(i)\) are assumed to be mutually independent. 
The probability assigned to an output \(o\) for input \(i\) is then
\begin{equation}
\begin{aligned}
p_\theta(o \mid i)
&:= \Pr\!\left(
    S(X_1(i),\ldots,X_m(i)) = o
    \right) \\
&= \sum_{\mathbf{x}:S(\mathbf{x})=o}
    \prod_{k=1}^{m}
    \pi_{\theta,k}(x_k \mid i),
\end{aligned}
\label{eq:theory:core}
\end{equation}
where \(\mathbf{x}=(x_1,\ldots,x_m)\).

Since $\Sfunction$ is deterministic and does not depend on $\theta$,
Eq.~\eqref{eq:theory:core} is a finite sum of products of neural
probabilities. Therefore, since the distributions
$\pi_{\theta,k}$ are differentiable with respect to $\theta$, then
$p_\theta(o \mid i)$ is differentiable as well.
This enables end-to-end learning from supervision on the final output,
without requiring labels for the intermediate variables $X_k(i)$.

\paragraph{Major computational challenge.}

Computing Eq.~\eqref{eq:theory:core} can nevertheless be expensive.
The set of assignments 
\[
    \bigl\{
        \mathbf{x} \in V_1 \times \cdots \times V_m
        \mid \Sfunction(\mathbf{x}) = o
    \bigr\}
\]
may be exponentially large in $m$.
Consequently, explicitly enumerating all assignments that contribute to the sum can become intractable.

The inference problem considered in this work is therefore to compute
$p_\theta(o \mid i)$ exactly, or to approximate it with certified bounds, without explicitly materializing the full set of satisfying assignments.

\section{DPNL: Oracle-Guided Inference}
\label{sec:dpnl}

We now introduce Dynamic Probabilistic Neurosymbolic Learning
(DPNL), an inference framework for computing the output probability
in Eq.~\eqref{eq:theory:core} without explicitly materializing the set of satisfying assignments. 
The key idea is to reason over \emph{regions} of the assignment
space rather than over complete assignments individually.

DPNL queries the symbolic component through an \emph{oracle} that,
for a region and a target output $o$, can return one of
three answers:
(i) every assignment in the region produces \(o\),
(ii) no assignment in the region produces \(o\), or
(iii) the available information is insufficient to decide.
In the first two cases, DPNL can account for or discard the entire
region without exploring its individual assignments. In the third case, DPNL further explores the region.

This oracle interface abstracts away the particular representation
of $\Sfunction$: DPNL only requires enough information to determine
whether a region can already be accepted or rejected.
The resulting split-or-prune process follows the same recursive
principle as DPLL-style search (an analogy intentionally reflected
in the name DPNL).
We now formalize these regions using valuations and state the conditions that an oracle must satisfy.

\subsection{Search Regions and Oracle Interface}

We formalize the regions explored by DPNL using \emph{valuations}.
A valuation $v$ of the variables $(X_k)_{1\leq k\leq m}$ is an
array of size $m$ such that, for every $k\in[1,m]$,
\[
    v[k] \in V_k \cup \{\unknown\}.
\]
The value $\unknown$ indicates that $X_k(i)$ has not yet been assigned.

We denote by $\setofvaluations{}$ the set of all valuations. A valuation $v$ is \emph{total} if
$v[k]\neq\unknown$ for every $k\in[1,m]$, and \emph{partial}
otherwise. 
A valuation $v'$ is a \emph{refinement} of $v$ if
\[
    \forall k\in[1,m],\qquad
    v[k]\neq\unknown \Longrightarrow v'[k]=v[k].
\]
Thus, $v'$ preserves all values fixed in $v$, while possibly fixing additional variables.

We denote by \(\tot{v}\) the set of all total refinements of \(v\).
We call \(\tot{v}\) the region represented by \(v\). %
The empty valuation $v_\emptyset = [\star,\ldots,\star]$
represents the whole assignment space, while a total valuation
represents a single complete assignment.

For $x\in V_k$, we denote by $v_{v[k]\leftarrow x}$ the valuation
obtained from $v$ by replacing $v[k]$ with $x$.

\label{section:oracle}

For a given valuation $v$ and output $o$ of $\Sfunction$, an \emph{oracle} is in charge of determining whether it is still relevant to explore refinements of $v$ or whether the exploration process can be stopped. 
Formally, an \emph{oracle} is a function
\[
    O_S :
    \setofvaluations{} \times \outputspace
    \rightarrow \{0,1,\unknown\}.
\]

Given a valuation $v$ and a target output $o$, the oracle returns
$1$ if it can establish that every total refinement of $v$ produces
$o$, $0$ if it can establish that no total refinement produces $o$,
and $\unknown$ otherwise. Intuitively, $\unknown$ indicates that the oracle cannot yet decide the region and further exploration is needed.

We call an oracle \emph{valid} when its definitive answers are sound
and it always decides total valuations.
Formally, $O_S$ is a valid
oracle for $\Sfunction$ if, for every
$v\in\setofvaluations{}$ and $o\in\outputspace$,
\begin{align}
    O_S(v,o)=1
    &\Longrightarrow
    \forall v' \in tot(v),\quad
    \Sfunction(v'[1],\ldots,v'[m]) = o,
    \label{eq:oracle:one}
    \\
    O_S(v,o)=0
    &\Longrightarrow
    \forall v' \in tot(v),\quad
    \Sfunction(v'[1],\ldots,v'[m]) \neq o.
    \label{eq:oracle:zero}
\end{align}
In addition, for every total valuation $v$, $O_S(v,o) \neq \unknown$.
Since $tot(v)=\{v\}$ when $v$ is total, the previous conditions imply
\[
    O_S(v,o)
    =
    \begin{cases}
        1 & \text{if } \Sfunction(v[1],\ldots,v[m]) = o,\\
        0 & \text{otherwise}.
    \end{cases}
\]

A valid oracle is \emph{complete} if the implications in
Eqs.~\eqref{eq:oracle:one}--\eqref{eq:oracle:zero} are equivalences.
Importantly, DPNL requires validity but not completeness:
returning $\unknown$ on a partial valuation only causes the
corresponding region to be explored further. Completeness therefore
affects pruning efficiency, but not correctness.

\subsection{Exact Inference with DPNL}

Algorithm~\ref{alg:DPNL} presents the core algorithm of the framework.  
\begin{algorithm}[htb]
\begin{small}
    \caption{DPNL($i,o,v,O_S$)}
    \label{alg:DPNL}
    \textbf{Input}: $i \in \inputspace$, $o \in \outputspace$, $v \in \setofvaluations{}$, and an oracle $O_S$ for $\Sfunction$. \\
\textbf{Output}: The probability of obtaining $o$ under $v$.\\
    \begin{algorithmic}[1]
        \IF {$O_S(v,o) = 1$}
            \STATE \textbf{return} $1$
        \ELSIF {$O_S(v,o) = 0$}
            \STATE \textbf{return} $0$
        \ELSE
            \STATE Choose $k \in [1,m]$ such that $v[k] = \unknown$. \label{alg:DPNL:choosek}
            \STATE \textbf{return} $\underset{y \in V_k}{\Sigma}\prob(X_k(i) = y)\label{alg:DPNL:useprob}\cdot\text{DPNL}(i,o,v_{v[k]\leftarrow y},O_S)$ 
        \ENDIF
    \end{algorithmic}
\end{small}
\end{algorithm}
Starting from the empty valuation
$v_\emptyset=[\unknown,\ldots,\unknown]$, DPNL computes the output probability in Eq.~\ref{eq:theory:core} as
\[
\begin{aligned}
    \mathrm{DPNL}(i,o,v_\emptyset,O_S)
    &=
    \prob\!\bigl(
        \Sfunction(X_1(i),\ldots,X_m(i)) = o
    \bigr)
    \\
    &=
    p_\theta(o\mid i).
\end{aligned}
\]

A recursive call terminates as soon as the oracle
returns $0$ or $1$. Otherwise, choosing an unassigned variable \(X_k(i)\)
partitions the current region into the regions represented by
its refinements, whose contributions are combined according to
\(\Pr(X_k(i)=y)\).

\begin{theorem}[DPNL termination and correctness]
\label{thm:dpnl}
For any input $i\in\inputspace$, output $o\in\outputspace$,
and valid oracle $O_S$ for $\Sfunction$,
$\mathrm{DPNL}(i,o,v_\emptyset,O_S)$ terminates and returns
$p_\theta(o\mid i)$.
\end{theorem}

The proof is provided in the supplementary material.

Crucially, DPNL explores the assignment space lazily: regions are further explored only when the oracle cannot decide them. It therefore avoids
materializing the full query provenance before inference.

\subsection{Approximate Inference with ApproxDPNL}

ApproxDPNL, presented in Algorithm~\ref{alg:ApproxDPNL}, extends
this lazy search with early termination while maintaining certified bounds on the exact probability.

\begin{algorithm}[htb]
\caption{ApproxDPNL$(i,o,O_S,\stopfunction,\searchheuristic)$}
\label{alg:ApproxDPNL}

\textbf{Input:}
$i\in\inputspace$, $o\in\outputspace$, a valid oracle $O_S$ for
$\Sfunction$, a stopping condition $\stopfunction$, and a search heuristic $\searchheuristic$.

\textbf{Output:}
Bounds $(\low,\up)$ such that
$\low \leq p_\theta(o\mid i) \leq \up$.

\begin{algorithmic}[1]
    \STATE $\low,\up \gets 0,1$
    \STATE $Q \gets \{v_\emptyset\}$

    \WHILE{$Q\neq\emptyset$ and $\neg\stopfunction(\low,\up)$}
        \STATE $v \gets \searchheuristic(Q)$
        \STATE $Q \gets Q\setminus\{v\}$

        \STATE $\mathit{mass} \gets
        \displaystyle\prod_{k:v[k]\neq\unknown}
        \pi_{\theta,k}(v[k]\mid i)$

        \IF{$O_S(v,o)=1$}
            \STATE $\low \gets \low+\mathit{mass}$

        \ELSIF{$O_S(v,o)=0$}
            \STATE $\up \gets \up-\mathit{mass}$

        \ELSE
            \STATE Choose $k\in[1,m]$ such that $v[k]=\unknown$
            \STATE $Q \gets Q\cup
            \{v_{v[k]\leftarrow y}\mid y\in V_k\}$
        \ENDIF
    \ENDWHILE 
    \RETURN $(\low,\up)$
\end{algorithmic}
\end{algorithm}

ApproxDPNL maintains a set \(Q\) of valuations representing regions
whose contribution to the target output has not yet been resolved.
Throughout the search, it maintains a lower bound
$\low$ and an upper bound $\up$ such that $p_\theta(o\mid i) \in[low,up]$. 
The lower bound $\low$ is the probability mass already certified
to produce $o$, while $\up$ additionally includes all unresolved
probability mass, which may still contribute to the output.
Equivalently, $1-\up$ is the probability mass already certified
not to produce $o$. Hence, $\up-\low$ is precisely the unresolved
probability mass, and at every iteration, the interval tightens:
$\low \leq p_\theta(o\mid i) \leq \up$.

\paragraph{Stopping condition.} A function $\stopfunction$ determines how tightly the returned
interval must bound $p_\theta(o\mid i)$ before inference terminates.
If $\stopfunction$ never triggers, all regions are eventually resolved and
$\low=\up=p_\theta(o\mid i)$, recovering exact inference.
Different stopping conditions provide different approximation guarantees.

\emph{Additive approximation.} 
For a desired absolute error $\epsilon>0$  with respect to the exact probability \(p_\theta(o\mid i)\), the search can stop when
$\up-\low\leq 2\epsilon$. The midpoint
$\hat p=(\low+\up)/2$ then satisfies
$|\hat p-p_\theta(o\mid i)|\leq\epsilon$.

\emph{Multiplicative approximation.}
For $\epsilon>0$, $\hat p$ is a multiplicative $\epsilon$-approximation
when
$p_\theta(o\mid i)/(1+\epsilon)\leq\hat p
\leq(1+\epsilon)p_\theta(o\mid i)$.
When $\low>0$, stopping when
$\up\leq\low(1+\epsilon)^2$ and taking
$\hat p=\sqrt{\low\up}$ provides this guarantee.
This interval-to-estimate conversion follows~\cite{WMC:Approx:Dubray2024}.

\emph{Anytime inference.}
Without prescribing an error tolerance, ApproxDPNL may instead be interrupted under any computational budget:
the returned interval remains certified at any time.
Its midpoint has absolute error at most $(\up-\low)/2$; when
$\low>0$, the geometric mean provides a multiplicative guarantee with
$\epsilon=\sqrt{\up/\low}-1$.

\paragraph{Search heuristic.}
The search heuristic $\searchheuristic$ determines which unresolved region is
explored next by selecting a valuation \(v\) from \(Q\).
The choice of the next unassigned variable \(X_k(i)\) can likewise
follow a fixed ordering or a problem-specific branching heuristic.
These choices affect how quickly the bounds tighten, but not their
validity.

\subsection{Designing and Combining Oracles}
\label{sec:dpnl:oracles}

Both DPNL and ApproxDPNL rely on the same oracle interface.
Any symbolic function $\Sfunction$ that can be evaluated on complete valuations admits a valid oracle: one may simply return $\unknown$ on every partial valuation and evaluate $\Sfunction$ once the valuation is total. 
Hence, the existence of an oracle does not restrict the applicability of DPNL. 
Oracles determine how efficiently the search space can be explored. 

In practice, oracle design involves a trade-off between the cost of
each oracle call and its \emph{pruning power}, i.e., how early and how
much of the remaining search space it can resolve. A stronger oracle
may avoid many recursive calls, but may itself be more expensive to
evaluate. In particular, completeness maximizes the ability to decide
partial valuations, but does not imply computational efficiency.

\paragraph{Basic oracles.}
When an executable implementation of $\Sfunction$ is available, a basic oracle can be obtained automatically by running it on a partial valuation. 
If the execution attempts to access an unassigned variable,
the oracle returns $\unknown$; otherwise, the computation terminates without depending on the remaining variables and its output determines whether to return $0$ or $1$. 
Basic oracles therefore inherit the computational efficiency of the implementation of $\Sfunction$, but need not be complete.

\paragraph{Solver-backed oracles.}
When $\Sfunction$ is expressed in a decidable logic, a solver can be used to determine whether all completions of the current partial valuation produce the target output $o$, or whether none of them do.
Such oracles can provide substantially stronger pruning than direct program execution. 
We give in the supplementary material a generic construction for logic-based PNL systems, together with sufficient conditions under which the resulting
oracle is valid and complete.

\paragraph{Custom and combined oracles.}
Domain knowledge can also be used to implement specialized valid oracles. For example, in relational reasoning tasks such as CLUTRR, an oracle can propagate the relations already fixed in a partial valuation. Knowing that $A$ is the sister of $B$ and $B$ the father of $C$,
for instance, may already determine that $A$ is the aunt of $C$, independently of the remaining unknown relations.

Different valid oracles can moreover be combined safely. Given $O_1,\ldots,O_p$, DPNL may query them in sequence and return the first answer different from $\unknown$. 
The resulting oracle remains valid;
if one of the queried oracles is complete, the combination is complete as well. 
This makes it possible to query inexpensive oracles first and fall back to stronger, more costly reasoning only when needed.

\section{Experimental Evaluation}

We conduct experiments on symbolic and neurosymbolic reasoning tasks in order to bring elements of answers to the following research questions:

\textbf{Q1:} How does DPNL compare to other exact probabilistic reasoning systems? 

\textbf{Q2:} In a neurosymbolic setting, how does the DPNL approach compare to SOTA PNL systems with exact and approximate reasoning in terms of efficiency and accuracy? 

\textbf{Q3:} How do the parameters of ApproxDPNL, the per-sample timeout of the anytime approximation and the error budget $\epsilon$ of the $\epsilon$-approximation, trade accuracy against runtime?

We first report on the experimental methodology and results for symbolic reasoning tasks, followed by neurosymbolic evaluations and an analysis of the $\epsilon$ and anytime timeout parameters of ApproxDPNL.
 
\subsection{Symbolic reasoning (Q1)}

We compare DPNL to ProbLog a state-of-the-art system with exact probabilistic inference used  in  DeepProbLog. 
We use the two following probabilistic reasoning tasks.

We first evaluate DPNL on the popular MNIST-N-SUM task. 
By extending from an N digit sum to an  N+1 digit sum, the complexity of reasoning increases exponentially. 
Thus, the multi-digit MNIST-N-SUM task is well-suited for exploring the scalability  of PNL systems. Each digit is represented as a probability distribution over $\{0, \ldots, 9\}$, simulating uncertainty in perception.
Here, we randomly generate the digit distributions. 
The symbolic function $S$ performs digit by digit integer addition and verifies the equality with the target.
The following oracles are compared on this task.
The \emph{basic oracle} is built on the digit-by-digit addition implementation for $S$.
The \emph{logic oracle} is generated from the Boolean logic representation of the N-digit addition.
The \emph{combined oracle} uses the sequential combination of the basic and logic oracles.

We also evaluate DPNL on a graph reachability task. This involves estimating the probability that a path exists between a source and a target node in a probabilistic graph.
We consider a graph with $N$ nodes, represented by an adjacency matrix containing $N^2$ potential edges.
Each edge is assumed to be independently present with a certain probability.
The symbolic function $S$ is a reachability predicate, implemented using a depth-first search (DFS) algorithm.
Probabilities are randomly assigned to the edges, and the existence of a path from the source to the destination node is evaluated.

\begin{figure}[htb]
        \centering
        \includegraphics[width=0.85\linewidth]{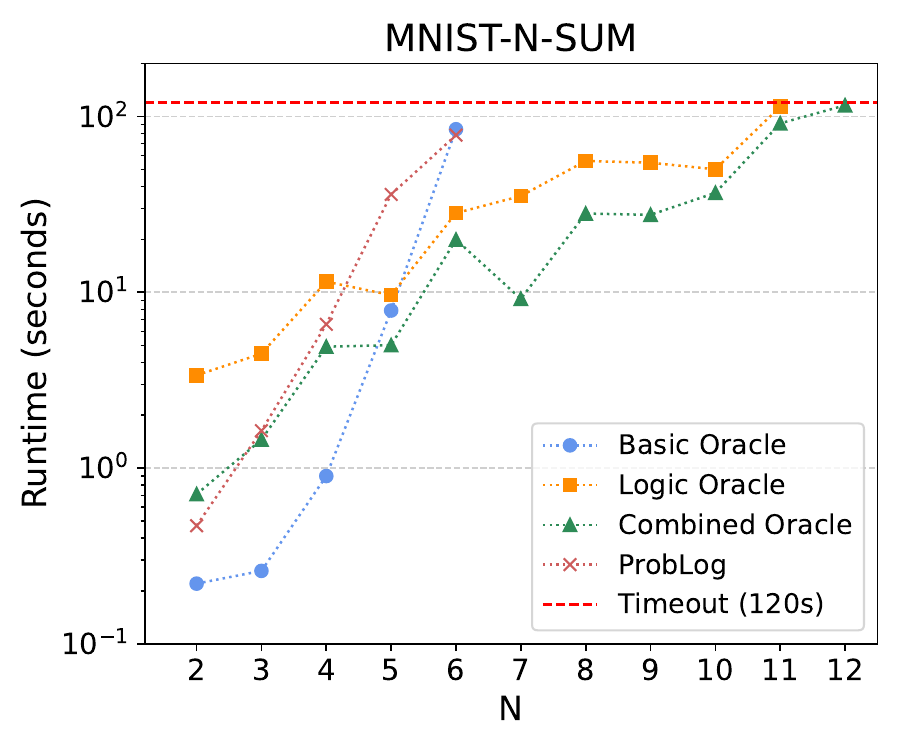}
        \vspace{-0.2cm}
        \caption{Q1 Runtimes (in seconds) for MNIST-N-SUM.}
    \label{fig:symbolic_runtimes:mnistnsum}
\end{figure}

\begin{figure}[htb]
    \centering
        \includegraphics[width=0.85\linewidth]{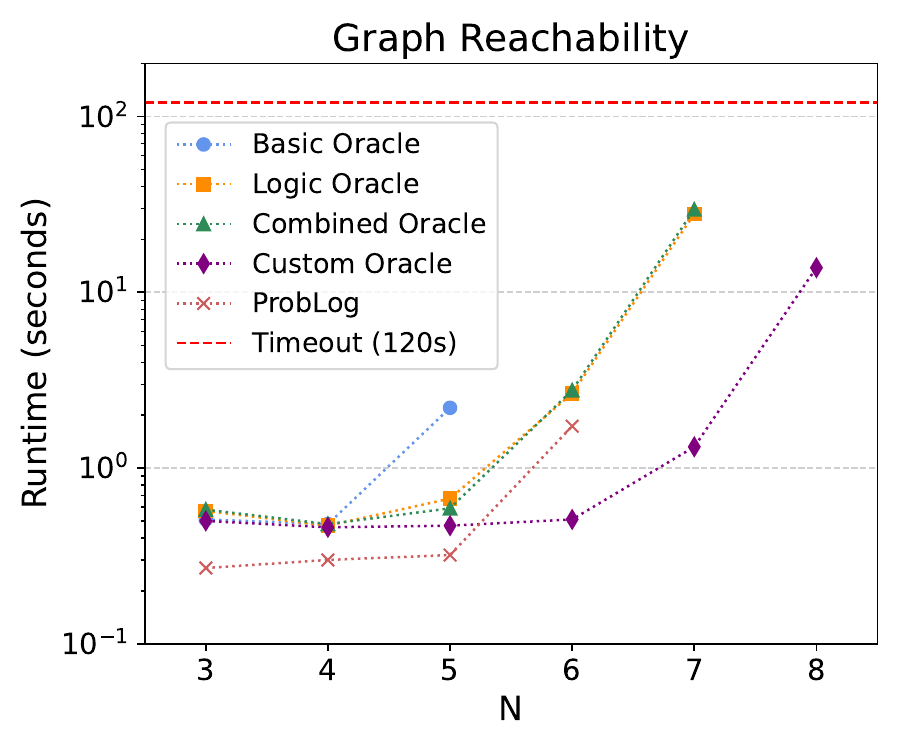}
    \vspace{-0.2cm}
    \caption{Q1 Runtimes (in seconds) for graph reachibility.}
    \label{fig:symbolic_runtimes:graphreach}
\end{figure}

\paragraph{Experimental results.}
Figures~\ref{fig:symbolic_runtimes:mnistnsum} and \ref{fig:symbolic_runtimes:graphreach}  show the time spent on symbolic reasoning for MNIST-N-SUM and graph reachability, respectively.
 On MNIST-N-SUM, results show that DPNL scales better than ProbLog, solving task instances for $N\geq 7$ where ProbLog fails with a timeout.
Results also show that the logic oracle, although slower on small instances, scales better than the basic oracle.

Regarding graph reachability, for small instances (up to $N=5$) Problog is faster than DPNL.
However, for $N\geq 7$, Problog fails with timeout, whereas DPNL successfully completes with the logic and combined oracles. 
Results also illustrate that DPNL offers the possibility to scale even further given some manual programming effort with a custom oracle.

\subsection{Application to neurosymbolic AI (Q2)}
We now evaluate DPNL on a neuro-symbolic benchmark suite, including MNIST-N-Sum, Visual Sudoku and Leaf classification. 
Table~\ref{tab:benchmark_overview} summarizes the difficulty of the task (regarding the symbolic search space) and the versions of the algorithm (exact, $\epsilon$-approximation or anytime approximation) that are evaluated on each benchmark. 
More details on hyperparameter values, timeout thresholds and information on hardware and baseline configurations are given in the supplementary material.

\begin{table}[t]
    \centering
    \small
    \setlength{\tabcolsep}{4pt}
    \renewcommand{\arraystretch}{1.25}
    \begin{tabular}{@{}ll ccc@{}}
        \toprule
        & & \multicolumn{3}{c}{\textbf{Reasoning}} \\
        \cmidrule(lr){3-5}
        \textbf{Benchmark} & \textbf{Difficulty}
            & Exact & \shortstack{$\epsilon$-\\approx.} & \shortstack{Anytime\\approx.} \\
        \midrule
        Leaf           & $120$                  & \checkmark &            &            \\
        MNIST-$N$-SUM   & $10^{2N}$              & \checkmark & \checkmark & \checkmark \\
        Visual Sudoku  & $10^{N\times N}$        &            &            & \checkmark \\
        \bottomrule
    \end{tabular}
    \caption{Overview of the neurosymbolic benchmarks.
    \emph{Difficulty} is the size of the symbolic search space
    The \emph{Reasoning} columns mark which DPNL inference modes are evaluated on each benchmark.
    }
    \label{tab:benchmark_overview}
\end{table}

\subsubsection{Leaf Classification}
We use the LEAF-ID dataset built from the plant-leaf collection of \citet{leaf:Kanda2021} which includes leaf images of $11$ plant species, with $330$ training and $110$ test samples.
Each leaf is described by three \emph{discrete} features (\textsc{margin}, \textsc{shape} and \textsc{texture}) with $6$, $5$ and $4$ possible values respectively.
A neural model predicts a distribution over the values of each feature from the image, and the symbolic function $S$ maps the feature values to a species.
$S$ is a fixed decision tree over the species. 
Because the tree has a bounded number of paths, the reasoning difficulty does not scale with any instance-size parameter. 
There are at most $6 \times 5 \times 4 = 120$ feature combinations.
The oracle evaluates the decision tree on a partial valuation and returns the species as soon as the branch is determined, usually without reading all three features.
DPNL is used with exact reasoning on this benchmark.
DPNL is trained for $50$ epochs with the Adam optimiser (learning rate $1\mathrm{e}{-4}$, batch size $4$), and all methods are evaluated over $10$ independent runs.
We compare DPNL against two families of baselines.
The \emph{neurosymbolic} baselines include DeepProbLog \cite{NeSy:DeepProbLog:Manhaeve2018}, Scallop \cite{NeSy:Scallop:Huang2021} in two configurations (top-$k$ provenance and exact provenance, A-NeSI \cite{NeSy:A-NeSI:vanKrieken2023}, and ISED \cite{NeSy:ISED:SolkoBreslin2024}.
The \emph{sampling-based gradient-approximation} baselines REINFORCE \cite{RL:REINFORCE:Williams1992} and NASR \cite{NeSy:NASR:Cornelio2023} treat the tree as a black box and estimate the gradient from sampled neural predictions
\paragraph{Results.}
The results are given in Table~\ref{tab:leaf}.
\begin{table}[t]
    \centering
    \small
    \setlength{\tabcolsep}{5pt}
    \renewcommand{\arraystretch}{1.2}
    \begin{tabular}{@{}lccc@{}}
        \toprule
        \textbf{Method} & \textbf{Accuracy} & \textbf{Std} & \textbf{Time} \\
        \midrule
        DPNL (exact)      & \textbf{0.8080} & 0.0360 & 10:15 \\
        \midrule
        \multicolumn{4}{@{}l}{\emph{Neurosymbolic (probabilistic inference)}} \\
        DeepProbLog       & 0.4327 & 0.0308 & 1:02:28 \\
        Scallop (top-$k$) & 0.7709 & 0.0284 & 5:30 \\
        Scallop (exact)   & 0.7691 & 0.0294 & 55:34 \\
        A-NeSI            & 0.7291 & 0.0909 & 21:33 \\
        ISED              & 0.7964 & 0.0261 & 18:03 \\
        \addlinespace[2pt]
        \multicolumn{4}{@{}l}{\emph{Sampling-based gradient approximation}} \\
        REINFORCE         & 0.7955 & 0.0357 & 6:52 \\
        NASR              & 0.1009 & 0.0155 & \textbf{2:03} \\
        \bottomrule
    \end{tabular}
    \caption{Results on the Leaf benchmark.
    Average test accuracy, standard deviation, and \emph{total} wall-clock runtime (\texttt{h:mm:ss}), over $10$ independent runs.
    Highest accuracy and lowest runtime in bold.
    }
    \label{tab:leaf}
\end{table}
DPNL with exact reasoning obtains the highest mean accuracy ($0.808$) on Leaf.
Among the exact methods it is also the fastest: $10$:$15$ against $55$:$34$ for Scallop with $k=120$ and $1$:$02$:$28$ for DeepProbLog, whose accuracy drops to $0.43$.
ISED and REINFORCE come close to DPNL ($0.796$), whereas NASR falls to chance level ($0.10 \approx 1/11$).

\subsubsection{MNIST-N-Sum}
MNIST-$N$-SUM \cite{NeSy:DeepProbLog:Manhaeve2018} is the task of predicting the sum of two $N$-digit numbers written as sequences of handwritten-digit images, supervised only by the value of the sum and never by the individual digit labels (e.g.\ a pair of two-digit images reading $28 + 35 = 63$).
We evaluate $N=\{2, 3, 4, 8\}$.
DPNL is run in three modes: exact inference, $\epsilon$-approximation, and anytime approximation.
Training uses Adam with learning rate $1\mathrm{e}{-3}$ and batch size $2$. 
The accuracy numbers are averaged over $10$ seeds.
Per-mode and per-$N$ hyper-parameters are listed in Table~\ref{tab:hyperparameters} in the appendix.
We compare DPNL (with the basic oracle) with the exact neurosymbolic methods DeepProbLog and Scallop \cite{NeSy:Scallop:Huang2021} and with A-Nesi \cite{NeSy:A-NeSI:vanKrieken2023}, DPLA$^*$ \cite{NeSy:DeepProbLog:Manhaeve2021} and Scallop (with the top-k semiring), which use approximate inference.
A-Nesi is evaluated across its different variants (predict only, explain, prune). 
We use the MNIST dataset from \cite{Article:LeCun:MNIST}.
Each experiment is performed with multiple independent runs, and we report the average test accuracy and the average reasoning time per sample during training.
The parameters for $\epsilon$ and timeout for ApproxDPNL have been chosen based on parameter tuning to find a good balance between accuracy and runtime, see Section~\ref{sec:ablation_studies}.
\paragraph{Results}
Among the other exact PNL systems, DeepProbLog matches DPNL's accuracy at $N=2$ but no longer finishes for $N \geq 3$, and Scallop (exact) already times out at $N=2$.
DPNL in exact mode is the only method that completes $N=4$ (avg.\ accuracy $0.83$) and times out only at $N=8$.
The $\epsilon$- and anytime-approximation modes preserve accuracy while improving scalability.
The anytime mode matches exact inference to within one point at every $N$ while cutting the per-sample symbolic time by $5$ to $23\times$ (e.g.\ $0.27$\,s vs.\ $6.1$\,s at $N=4$), and it is the only method that runs at all at $N=8$ within the timeout threshold. 
The $\epsilon$-approximation behaves similarly.
It stays within one point of exact inference ($0.9062$, $0.8730$, $0.8219$ for $N = 2,3,4$), and at $N=4$ its per-sample symbolic time is $0.67$\,s which is $9\times$ below exact, between the anytime mode ($0.27$\,s) and exact ($6.1$\,s).
Comparing exact and approximate inference under the same training budget, a clear accuracy gap opens between the DPNL modes and the approximate baselines from $N=3$ on.
DPNL reaches $\approx 0.87$ at $N=3$ versus at most $0.84$ for the baselines, and at $N=4$ only A-NeSI among the baselines finishes, its best variant (explain) reaching $0.66$ versus $0.83$ for DPNL.
This gap should be read as a difference in data efficiency under a short training budget rather than as a bound on the accuracy that approximate methods can reach, since A-NeSI is reported to reach higher accuracies with substantially more training epochs \cite{NeSy:A-NeSI:vanKrieken2023}.
At $N=8$, all three A-NeSI variants terminate and reach $0.702 \pm 0.025$ (predict), $0.702 \pm 0.031$ (explain) and $0.684 \pm 0.021$ (prune), close to the $0.723 \pm 0.257$ of anytime DPNL, the only DPNL mode that finishes at this scale.
A-NeSI keeps its speed advantage: its per-sample reasoning time is $0.0136$\,s (predict), $0.0349$\,s (explain) and $0.0340$\,s (prune), the same order of magnitude as for $N \le 4$ and more than $50\times$ below the $1.94$\,s of anytime DPNL. 

\begin{table*}[t]
    \centering
    \renewcommand{\arraystretch}{1.2}
    \resizebox{\textwidth}{!}{
    \begin{footnotesize}
    \begin{tabular}{|l||c|c|c||c|c|c||c|c|c||c|c|c|}
    \hline
    \textbf{} & \multicolumn{3}{c||}{\textbf{$\mathbf{N=2}$}} & \multicolumn{3}{c||}{\textbf{$\mathbf{N=3}$}} & \multicolumn{3}{c||}{\textbf{$\mathbf{N=4}$}} & \multicolumn{3}{c|}{\textbf{$\mathbf{N=8}$}} \\
    \hline
                    & Acc & Std & Time & Acc & Std & Time & Acc & Std & Time & Acc & Std & Time \\
    \hline
    \textbf{DPNL (exact)}        & 0.9100 & 0.0161 & 0.0301 & \textbf{0.8735} & 0.0168 & 0.4783 & \textbf{0.8320} & 0.0209 & 6.1344 & T/O    & T/O    & T/O \\
    \textbf{DPNL ($\epsilon$)}   & 0.9062 & 0.0124 & 0.0061 & 0.8730 & 0.0216 & 0.0430 & 0.8219 & 0.0243 & 0.6656 & T/O    & T/O    & T/O \\
    \textbf{DPNL (anytime)}      & \textbf{0.9115} & 0.0118 & 0.0063 & 0.8636 & 0.0157 & 0.0271 & 0.8266 & 0.0367 & 0.2654 & \textbf{0.7227} & 0.2566 & 1.9446 \\
    \hline
    \hline
    \textbf{DeepProbLog}      & 0.9098 & 0.0122 & 0.4784 & T/O & T/O & T/O & T/O & T/O & T/O & T/O & T/O & T/O \\
    \textbf{Scallop (exact)}  & T/O & T/O & T/O & T/O & T/O & T/O & T/O & T/O & T/O & T/O & T/O & T/O \\
    \hline
    \textbf{DPL A*}  & 0.8945 & 0.0146 & 0.0255 & 0.8365 & 0.0303 & 0.1237 & T/O & T/O & T/O & T/O & T/O & T/O \\
    \textbf{Scallop (top-k)}  & 0.8257 & 0.2060 & 0.0244 & 0.7523 & 0.2612 & 0.9521 & T/O & T/O & T/O & T/O & T/O & T/O \\
    \textbf{A-Nesi (predict)}$^{\dagger}$  & 0.8849 & 0.0142 & 0.0265 & 0.7796 & 0.0074 & \textbf{0.0170} & 0.5929 & 0.0394 & \textbf{0.0221} & 0.7023 & 0.0248 & \textbf{0.0136} \\
    \textbf{A-Nesi (explain)}$^{\dagger}$  & 0.8952 & 0.0071 & 0.0213 & 0.7707 & 0.0179 & 0.0219 & 0.6584 & 0.0239 & 0.0268 & 0.7016 & 0.0309 & 0.0349 \\
    \textbf{A-Nesi (prune)}$^{\dagger}$  & 0.7866 & 0.2225 & \textbf{0.0195} & 0.7706 & 0.0214 & 0.0206 & 0.5905 & 0.0339 & 0.0245 & 0.6841 & 0.0207 & 0.0340 \\
    \hline
    \end{tabular}
    \end{footnotesize}
    }
    \caption{Results on the MNIST-$N$-SUM task: average test accuracy (Acc), its standard deviation (Std), and average reasoning time per sample during training (Time, in seconds), over several independent runs.
    Highest accuracy and lowest runtime in bold.
    $^{\dagger}$It is worth noting that \citet{NeSy:A-NeSI:vanKrieken2023} report higher accuracies for A-NeSI ($0.960$ for $N=2$ and $0.926$ for $N=4$) with $100$ training epochs and batch size $16$ on a GPU.
    We instead train A-NeSI with training parameters similar to those of the other methods, to keep accuracies and reasoning times comparable across systems.
    For $N=8$, all methods are trained for more epochs, which explains why A-NeSI reaches a higher accuracy at $N=8$ than at $N=4$.
    }
    \label{tab:performance_comparison}
    \vspace{-0.3cm}
\end{table*}

\subsubsection{Visual Sudoku}
We use the Visual Sudoku Puzzle Classification dataset of \citet{NeSy:ViSudo:Augustine2022}.
The task is to decide whether an $N \times N$ grid of digit images forms a \emph{valid} Sudoku, that is, whether no digit occurs twice in any row, any column, or any $\sqrt{N}\times\sqrt{N}$ block. 
We consider $N \in \{4, 9\}$.
Each cell is perceived by the same MNIST classifier as in the MNIST-$N$-SUM task.
In DPNL, each of the $N^2$ cells is a random variable over $\{0,\dots,9\}$, so a naive PWMC enumeration of the possible grids would consider $10^{N\times N}$ total valuations.
The oracle prunes this search space as follow. 
Given a partial valuation (a partially filled grid), it returns $0$ (invalid) as soon as a duplicate appears in any row, column or block, without waiting for the remaining cells to be assigned.
We test DPNL in its exact, anytime and epsilon mode. 
We compare DPNL to a CNN as neural baseline and to the neuro-symbolic methods NeuPSL \cite{Pyror2023} and A-Nesi. 
As for MNIST, the parameters for epsilon and anytime result from parameter tuning, see Section~\ref{sec:ablation_studies}. 
For N9, we use a pretrained classifier. 
For better comparability, we also evaluate A-Nesi with the same pretrained classifier on N9. 
The results for the CNN, A-Nesi (without pretraining) and NeuPSL are taken from \cite{NeSy:A-NeSI:vanKrieken2023}.
\paragraph{Results.}
Table~\ref{tab:vsudoku} reports the results.
At $N=4$, exact DPNL reaches $0.912$ without classifier pretraining, slightly above the best baselines trained in the same setting (A-NeSI, $0.898$; NeuPSL, $0.897$).
Both approximations stay close to exact inference ($0.910$ for the $\epsilon$-approximation and $0.890$ for anytime) while reducing the wall-clock time per run from about $45$ to $12$ and $5$\,min, respectively.
At $N=9$, the search space grows to $10^{81}$ grids and the exact inference baseline times out, whereas anytime DPNL still trains and reaches $0.757$, the highest accuracy of all methods, ahead of pretrained A-NeSI ($0.723$) that also profits from the pretraining of the classifier.

\begin{table}[t]
    \centering
    \footnotesize
    \setlength{\tabcolsep}{4pt}
    \begin{tabular}{@{}lcc@{}}
        \toprule
        \textbf{Method} & \textbf{$N=4$} & \textbf{$N=9$} \\
        \midrule
        DPNL (exact)               & $\mathbf{0.9115 \pm 0.0111}$ & T/O                  \\
        DPNL ($\epsilon$-approx.)  & $0.9100 \pm 0.0191$ & T/O                  \\
        DPNL (anytime)             & $0.8895 \pm 0.0297$ & $\mathbf{0.7570 \pm 0.0172}$ \\
        \midrule
        CNN                        & $0.5150 \pm 0.0334$ & $0.5120 \pm 0.0220$ \\
        NeuPSL \cite{Pyror2023}    & $0.8970 \pm 0.0220$ & $0.5150 \pm 0.0137$ \\
        A-NeSI (symbolic pred.)    & $0.8970 \pm 0.0208$ & $0.6215 \pm 0.0208$ \\
        A-NeSI (neural pred.)      & $0.8980 \pm 0.0211$ & $0.6225 \pm 0.0220$ \\
        A-NeSI (pretrained)        & -- & $0.7230 \pm 0.0214$ \\
        \bottomrule
    \end{tabular}
    \caption{Results on Visual Sudoku: mean test accuracy $\pm$ standard deviation over $10$ runs.
    Highest accuracy in bold.
    Baselines are taken from \citet{NeSy:A-NeSI:vanKrieken2023}, except A-NeSI (pretrained), which was ran on the same hardware.
    T/O denotes timeout.
    Wall-clock time per DPNL run is about $45$, $12$ and $5$\,min at $N=4$ (exact, $\epsilon$-approx., anytime) and $1$h\,43m to $3$h\,33m at $N=9$.}
    \label{tab:vsudoku}
\end{table}

\subsection{Analysis of the $\epsilon$ and Anytime Timeout Parameters (Q3)}
\label{sec:ablation_studies}
We study how the two parameters of ApproxDPNL trade accuracy against runtime and consider the per-sample timeout $t$ of the anytime approximation and the error budget $\epsilon$ of the $\epsilon$-approximation.

\paragraph{MNIST-$N$-SUM.}
Figure~\ref{fig:appendix:anytime-tuning} shows test accuracy and total runtime as a function of $t$ for $N \in \{2,3,4\}$.
Accuracy behaves like a threshold: it stays at $0$ as long as $t$ is too small for most samples to be resolved, then jumps to a plateau (about $0.90$ for $N=2$, $0.84$ to $0.89$ for $N=3$ and $0.83$ to $0.85$ for $N=4$) that larger timeouts barely improve.
Runtime, in contrast, grows roughly linearly with $t$, e.g.\ from $1$h\,20m at $t=0.5$\,s to $7$h\,57m at $t=3$\,s for $N=4$.
The best operating point therefore lies just above the threshold.
Figure~\ref{fig:appendix:epsilon-tuning} shows the analogous experiment over $\epsilon$ for $N \in \{1,2,3,4\}$.
Runtime drops steeply as $\epsilon$ grows, while accuracy stays close to that of exact inference up to a breaking point that decreases with the task size (about $\epsilon=0.6$ for $N=2$, $0.5$ for $N=3$ and $0.4$ for $N=4$) and collapses beyond it.
At $N=4$, exact inference and all budgets $\epsilon \le 0.2$ exceed the $4$\,h limit, so the $\epsilon$-approximation is what makes this instance tractable.
Larger instances thus tolerate less approximation, but they are also those that benefit most from it.

\begin{figure*}[t]
    \centering
    \begin{subfigure}{0.32\textwidth}
        \centering
        \includegraphics[width=\linewidth]{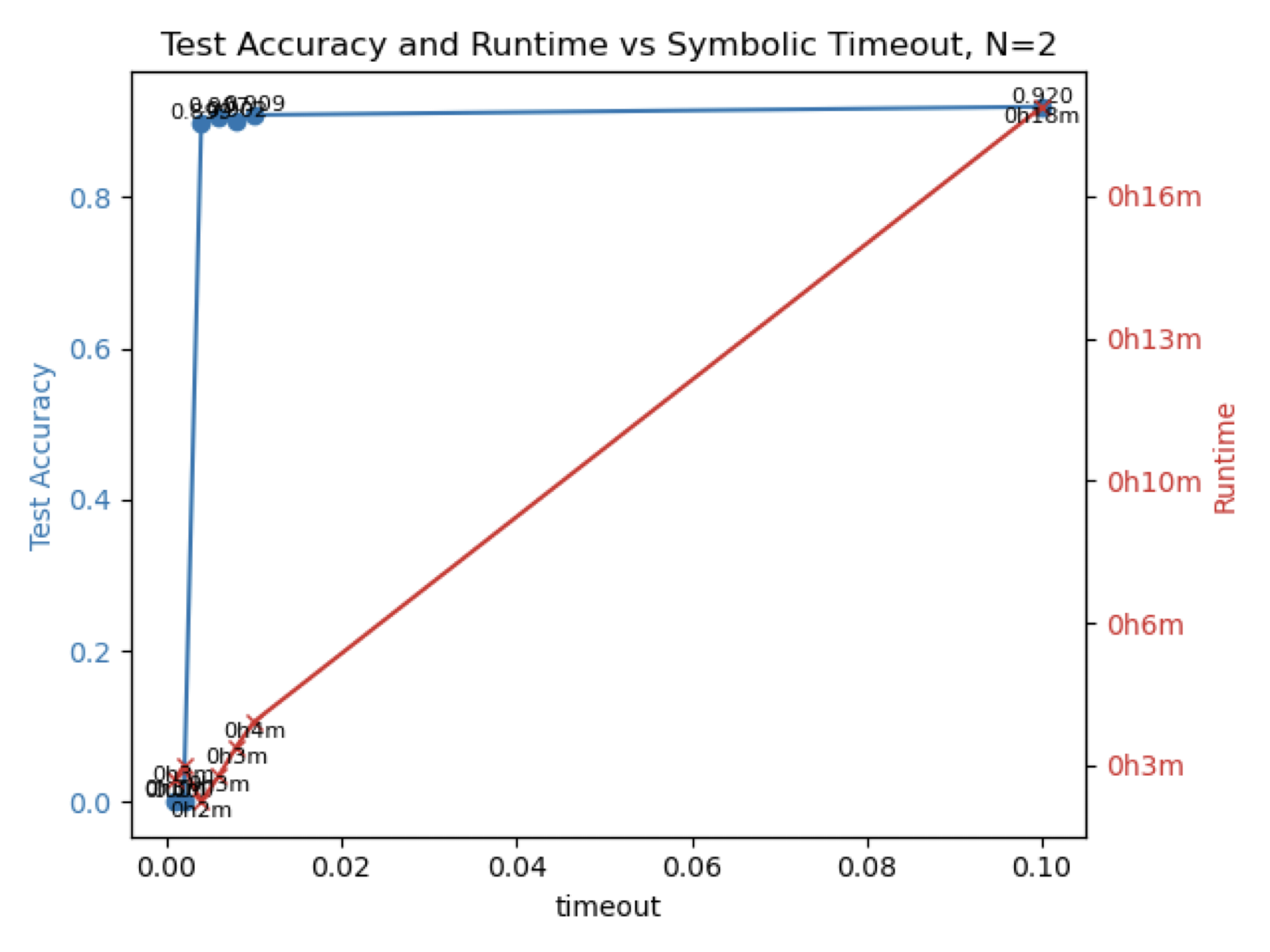}
        \caption{$N=2$}
    \end{subfigure}\hfill
    \begin{subfigure}{0.32\textwidth}
        \centering
        \includegraphics[width=\linewidth]{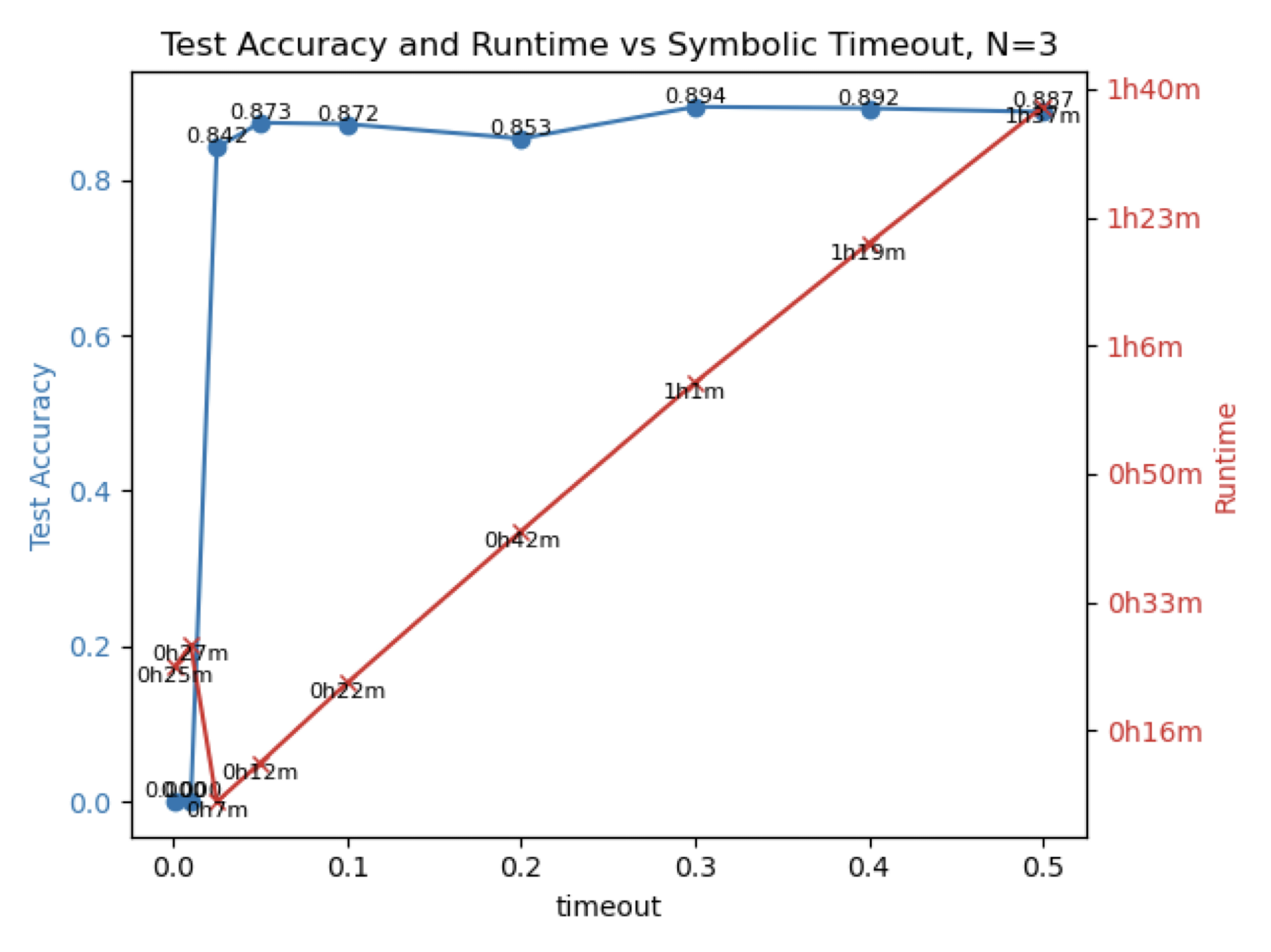}
        \caption{$N=3$}
    \end{subfigure}\hfill
    \begin{subfigure}{0.32\textwidth}
        \centering
        \includegraphics[width=\linewidth]{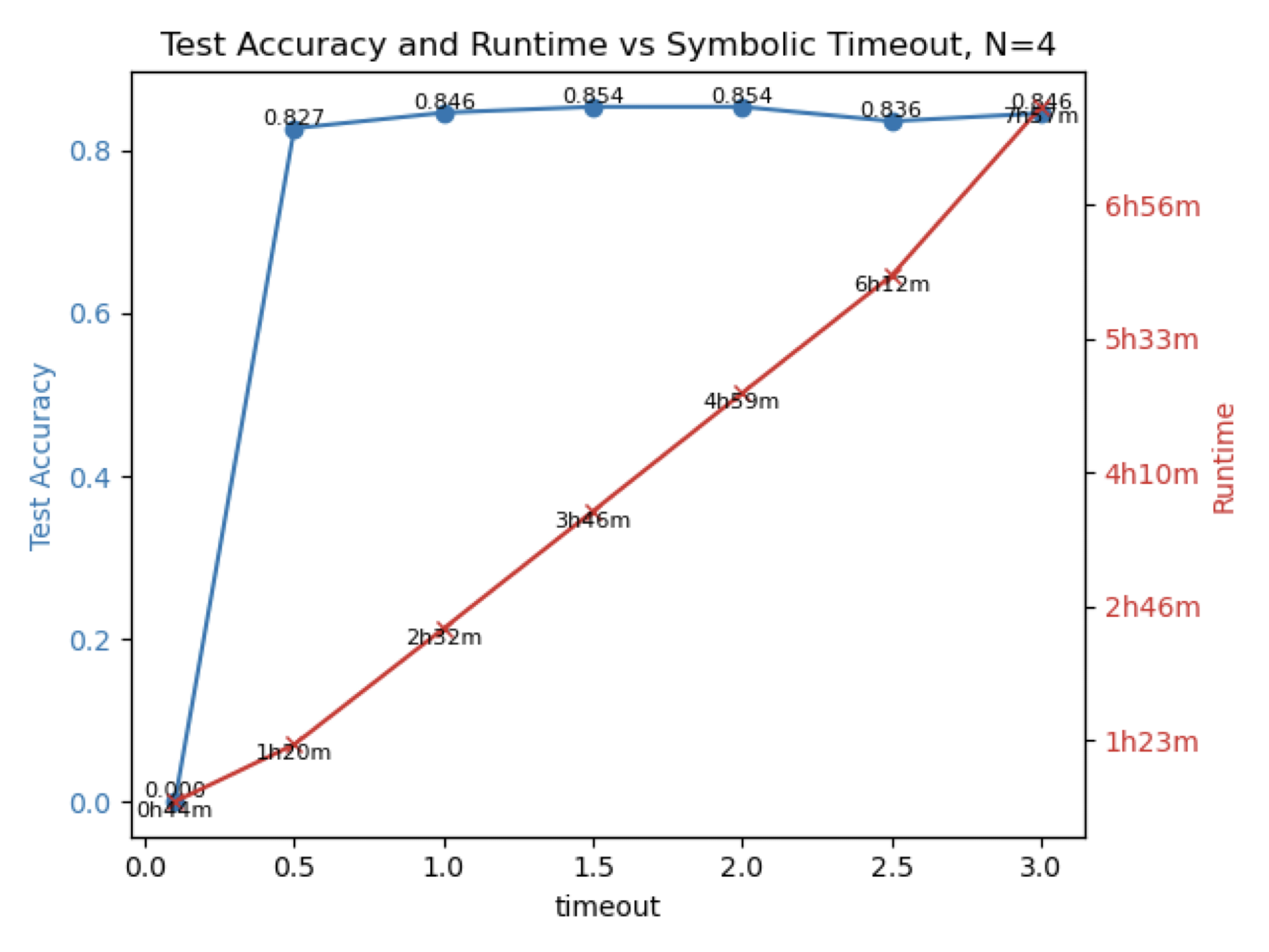}
        \caption{$N=4$}
    \end{subfigure}
    \caption{Test accuracy (blue, left axis) and total runtime (red, right axis) as a function of the per-sample symbolic timeout $t$ (in seconds) for the anytime approximation on MNIST-$N$-SUM.
    No sweep was run for $N=8$.}
    \label{fig:appendix:anytime-tuning}
\end{figure*}

\begin{figure*}[t]
    \centering
    \begin{subfigure}{0.49\textwidth}
        \centering
        \includegraphics[width=\linewidth]{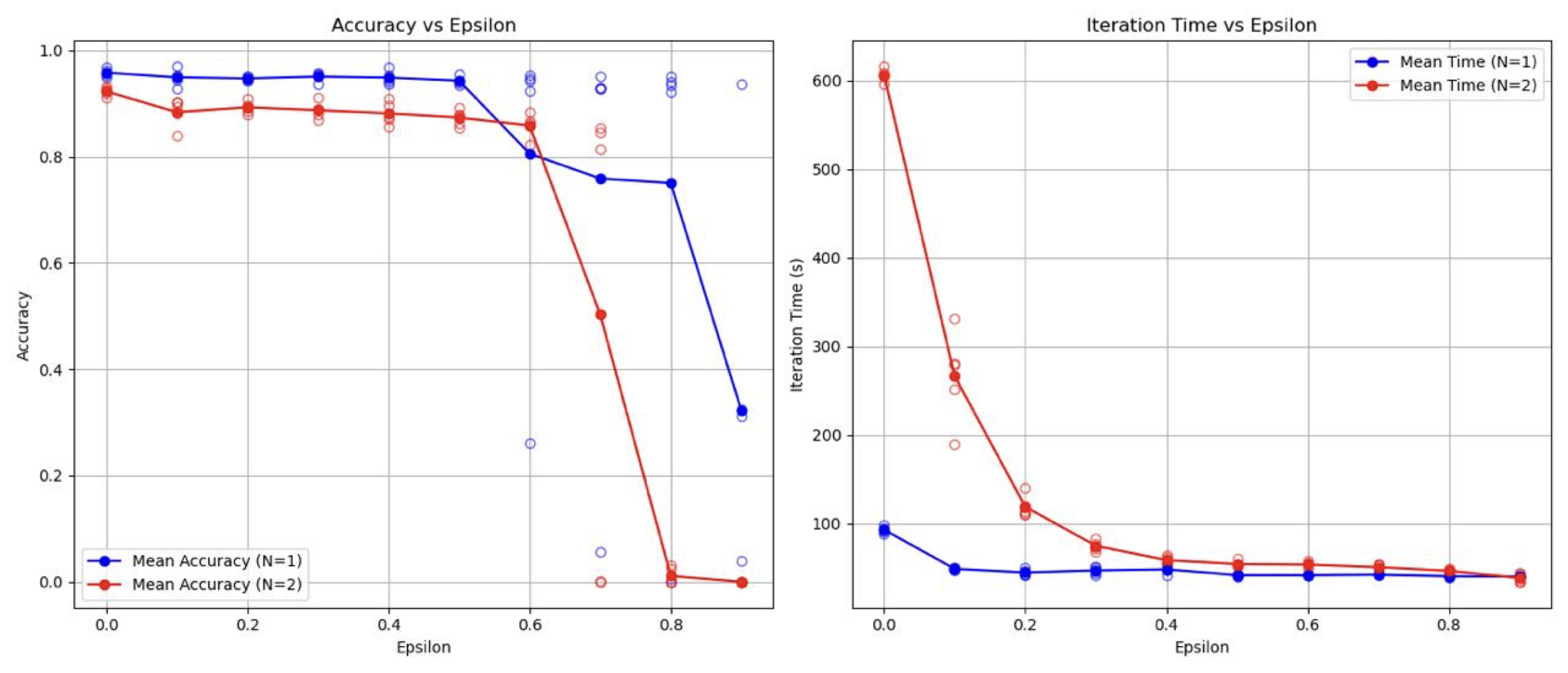}
        \caption{$N=1$}
    \end{subfigure}\hfill
    \begin{subfigure}{0.49\textwidth}
        \centering
        \includegraphics[width=\linewidth]{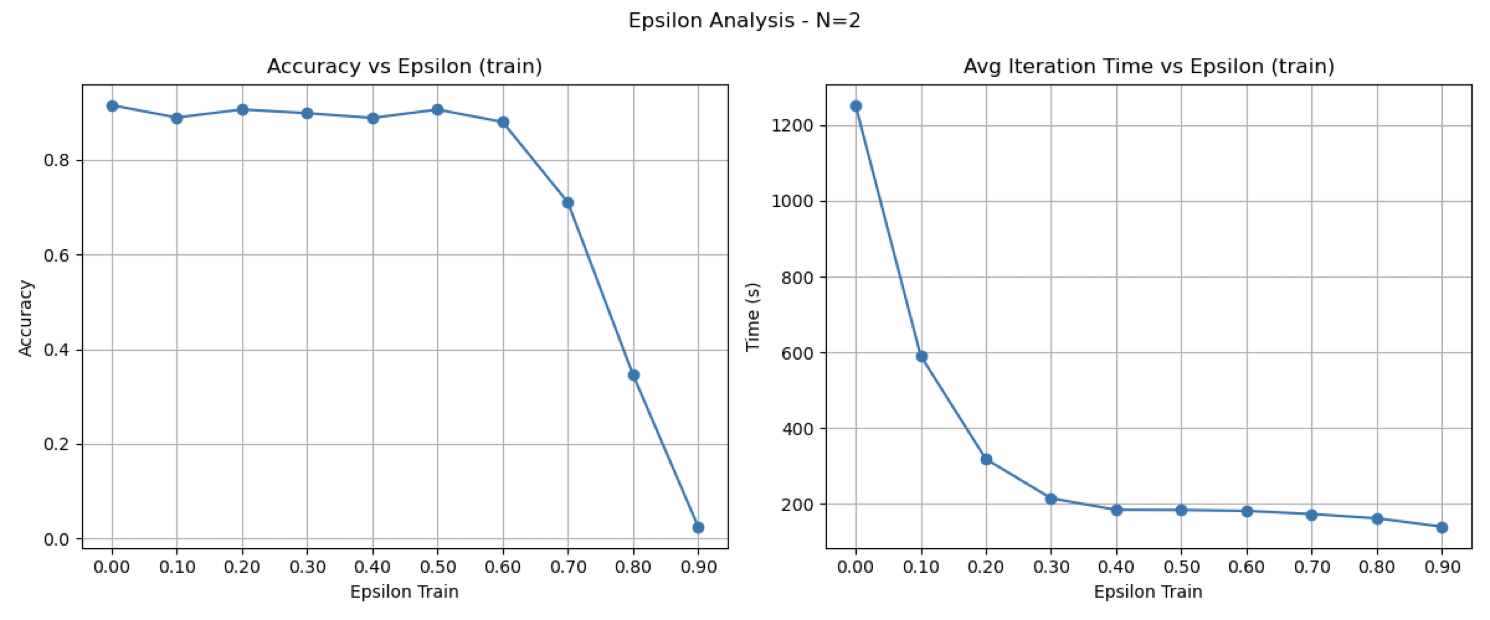}
        \caption{$N=2$}
    \end{subfigure}

    \vspace{0.4em}
    \begin{subfigure}{0.49\textwidth}
        \centering
        \includegraphics[width=\linewidth]{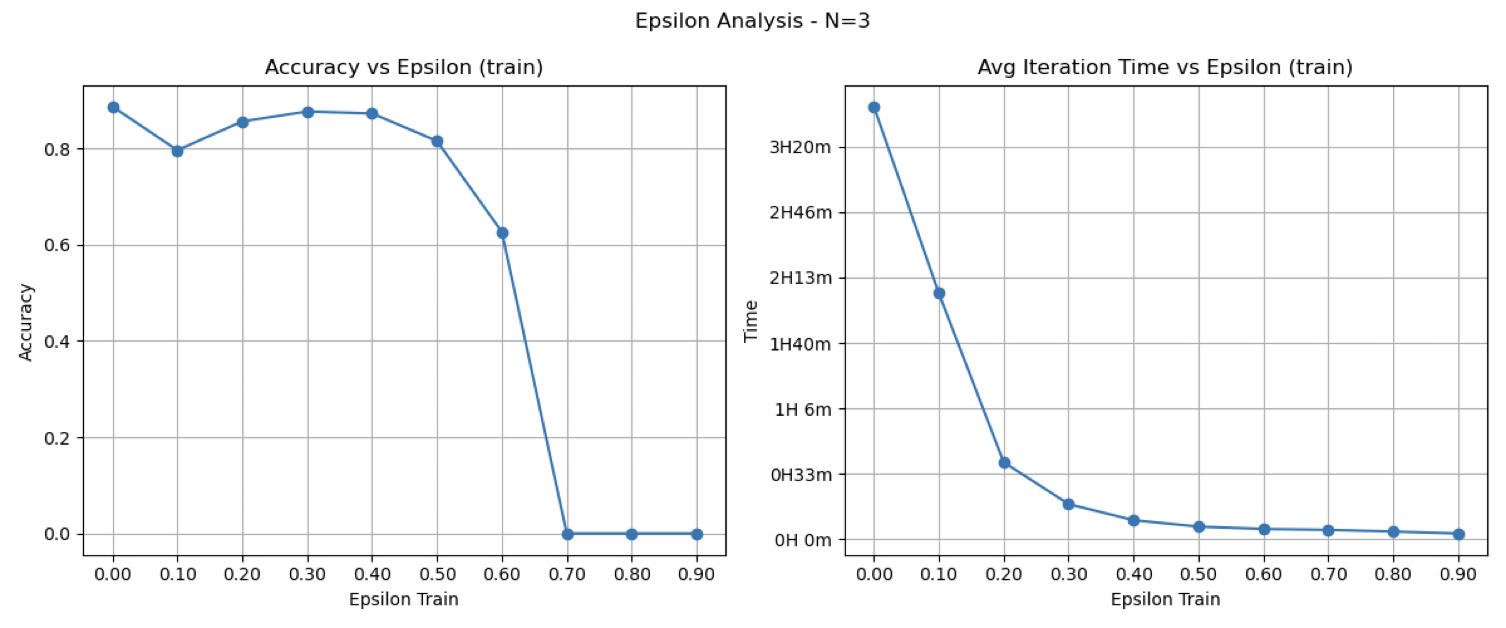}
        \caption{$N=3$}
    \end{subfigure}\hfill
    \begin{subfigure}{0.49\textwidth}
        \centering
        \includegraphics[width=\linewidth]{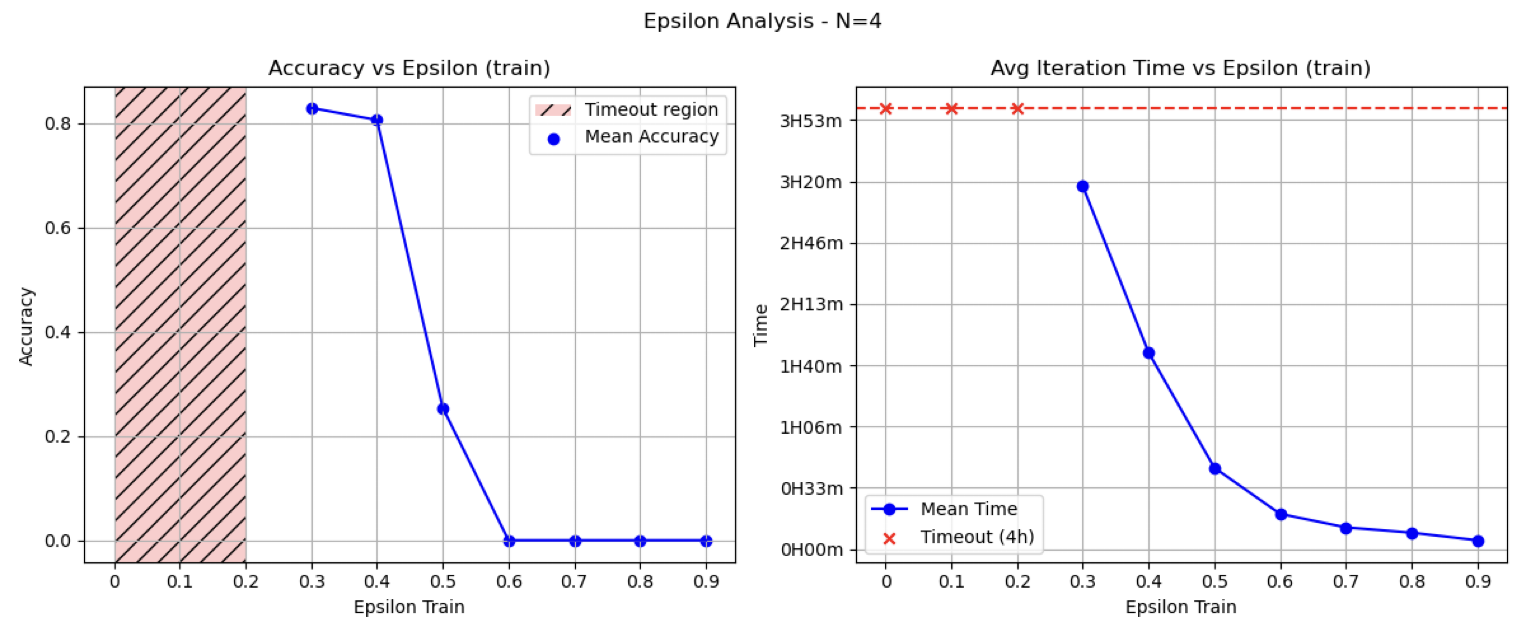}
        \caption{$N=4$}
    \end{subfigure}
    \caption{Test accuracy (left panel) and training time (right panel) as a function of the error budget $\epsilon$ for the $\epsilon$-approximation on MNIST-$N$-SUM.
    For $N=4$, the hatched region marks the budgets whose runs exceed the $4$\,h limit.
    No sweep was run for $N=8$.}
    \label{fig:appendix:epsilon-tuning}
\end{figure*}

\paragraph{Visual Sudoku.}
Table~\ref{tab:appendix:vsudoku-anytime-tuning} and Figure~\ref{fig:appendix:vsudoku-anytime-tuning} report the same observation over $t$ for $N=4$, with $10$ training epochs, batch size $1$, no classifier pretraining and $5$ seeds per configuration.
The anytime approximation shows the same threshold behaviour as on MNIST-$N$-SUM.
Accuracy stays at or below chance level ($0.5$ on this binary task) up to $t=0.06$\,s, rises to $0.815$ at $t=0.07$\,s and $0.880$ at $t=0.08$\,s, and then saturates between $0.877$ and $0.903$ up to $t=0.5$\,s, while the runtime keeps growing from $5$ to $16$\,min.
We therefore use $t = 0.08$\,s in the main experiments.

\begin{table}[t]
    \centering
    \footnotesize
    \setlength{\tabcolsep}{5pt}
    \begin{tabular}{@{}lccc@{}}
        \toprule
        \textbf{$t$ (s)} & \textbf{Acc.} & \textbf{Std} & \textbf{Runtime} \\
        \midrule
        0.01  & 0.350 & 0.189 & \textbf{1:07} \\
        0.02  & 0.401 & 0.085 & 1:52 \\
        0.04  & 0.492 & 0.018 & 2:56 \\
        0.05  & 0.499 & 0.002 & 3:35 \\
        0.06  & 0.534 & 0.076 & 4:16 \\
        0.07  & 0.815 & 0.022 & 4:49 \\
        0.075 & 0.839 & 0.037 & 5:03 \\
        0.08  & 0.880 & 0.046 & 5:13 \\
        0.1   & 0.877 & 0.025 & 6:23 \\
        0.2   & \textbf{0.903} & 0.019 & 9:19 \\
        0.5   & 0.884 & 0.037 & 16:16 \\
        \bottomrule
    \end{tabular}
    \caption{Anytime-timeout tuning on Visual Sudoku $N=4$: mean test accuracy, its standard deviation and mean total runtime (\texttt{mm:ss}) over $5$ seeds.
    Highest accuracy and lowest runtime in bold.}
    \label{tab:appendix:vsudoku-anytime-tuning}
\end{table}

\begin{figure}[t]
    \centering
    \includegraphics[width=\linewidth]{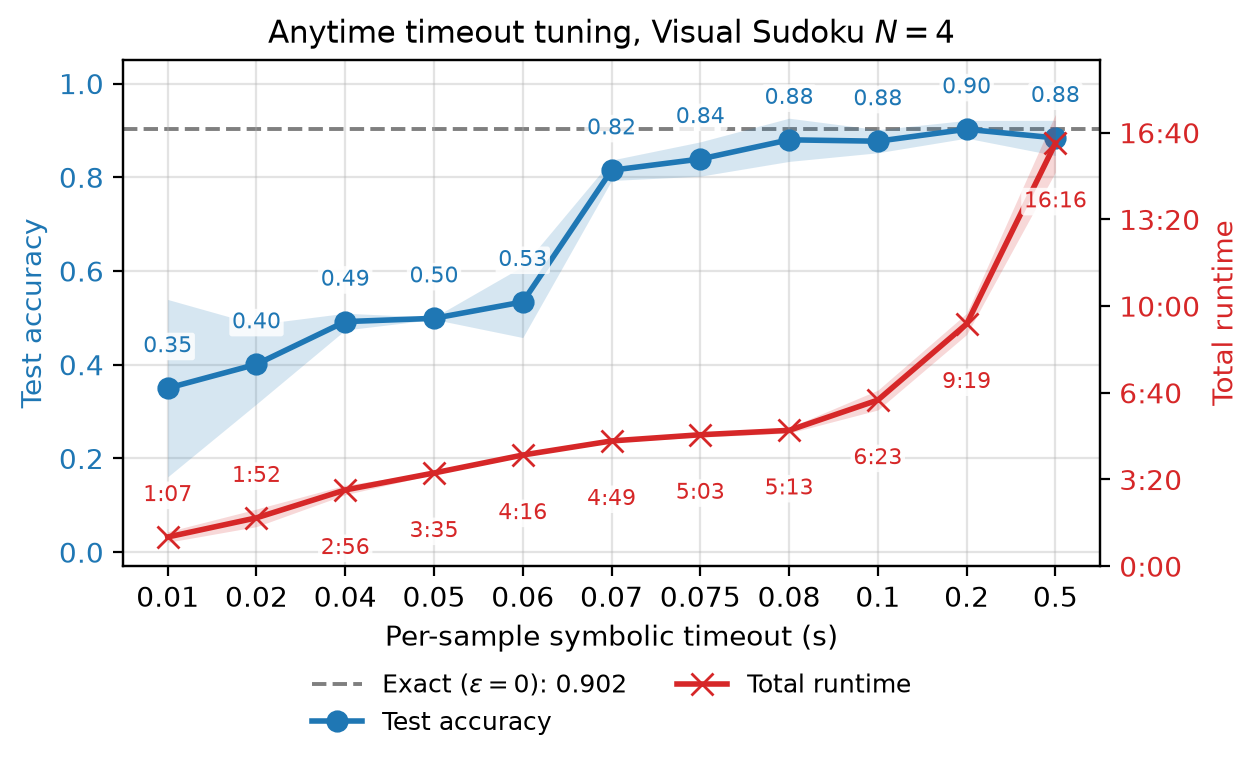}
    \caption{Test accuracy (blue, left axis) and total runtime (red, right axis) as a function of the per-sample timeout $t$ for the anytime approximation on Visual Sudoku $N=4$.
    Shaded bands span $\pm 1$ standard deviation over seeds; the dashed line marks exact inference (accuracy $0.902$).}
    \label{fig:appendix:vsudoku-anytime-tuning}
\end{figure}

\section{Related Works}
\label{sec:related_work}

\subsection{Symbolic AI}
\label{sec:rw:symbolic}
DPNL builds on a line of work on exact and approximate weighted model counting (WMC).
Probabilistic logic programming languages such as
ProbLog~\cite{NeSy:ProbLog:DeRaedt2007} provide an important foundation
for provenance-based probabilistic inference, later adopted by
neurosymbolic approaches such as DeepProbLog.
The recursive, search-based structure of DPNL follows the DPLL algorithm~\cite{WMC:DPLL:Davis1962}, which decomposes a formula by branching on variable assignments.
The use of oracles inside DPNL used to resolve a query under partial valuations is inspired by the decision-DNNF compilation strategies of Lagniez and Marquis~\cite{WMC:KC:LagniezMarquis2017}.
The approximation techniques used in ApproxDPNL are inspired by works in the symbolic AI literature, namely the anytime DPLL-based weighted model counting procedure of Dubray et al.~\cite{WMC:Approx:Dubray2024} and the approximate weighted model counting with $\epsilon$ guarantees~\cite{WMC:Approx:Chakraborty2015}.
Unlike DPNL and ApproxDPNL, none of these systems are used to train a neural component from output-level supervision.
They only address general-purpose probabilistic inference and compilation rather than neuro-symbolic learning.

\subsection{Neurosymbolic AI}
\label{sec:rw:neurosymbolic}
Neurosymbolic AI is a wide field containing a large number of different approaches.
They share the same overall goal: bringing neural networks and symbolic AI together, but differ widely in how the two are combined, e.g.\ by relaxing logical connectives into differentiable operators (fuzzy-logic approaches such as Logic Tensor Networks~\cite{Badreddine_2022}) or by adding a differentiable penalty for constraint violation to the training loss (the semantic loss~\cite{pmlr-v80-xu18h}).
Other probabilistic neurosymbolic frameworks adopt different inference
semantics altogether. For instance, NeuPSL~\cite{Pyror2023} integrates
neural predictions with Probabilistic Soft Logic and performs symbolic
inference as convex MAP inference in a deep hinge-loss Markov random
field.
See~\cite{NeSy:Survey:vanHarmelen2019,NeSy:Survey:DeRaedt2020,NeSy:Survey:Varmeulen2023} for surveys.

Within this broad landscape, we focus on approaches for solving the PNL problem defined in Section 2, in which neural predictions over intermediate variables are combined by a symbolic function to determine the final output. Existing approaches differ primarily in how they perform the probabilistic inference required to compute or approximate the output probability in Eq. (1). We distinguish below approaches based on exact reasoning, proof-based approximation, surrogate or sampling-based approximation, and hardware-oriented acceleration.

\paragraph{Exact reasoning.}
DeepProbLog~\cite{NeSy:DeepProbLog:Manhaeve2018} is the seminal PNL approach that compiles the query's full logical provenance and performs exact weighted model counting (WMC) on it.
NeurASP~\cite{NeurASP} provides an analogous integration with Answer Set Programs.
DPNL instead computes the output probability in Eq.~\ref{eq:theory:core} through a lazy search over regions of the assignment space, without requiring the complete provenance to be materialized before inference.

\paragraph{Proof-based approximation.}
A common approach to improving scalability is to restrict probabilistic inference to a subset of the query's proofs.
Scallop's top-$k$ semiring~\cite{NeSy:Scallop:Huang2021,NeSy:Scallop:Siyang2023} retains at most $k$ high-probability proofs during provenance construction, while approximate DeepProbLog~\cite{NeSy:DeepProbLog:Manhaeve2021} (DPLA$^*$) uses A$^*$-search to focus inference on high-probability explanations rather than exhaustively aggregating over all proofs.
These approaches therefore control the amount of proof search performed, rather than directly controlling the approximation error on the output probability.
In particular, Scallop's parameter $k$ controls the number of retained proofs, while the resulting error depends on the probability mass of the discarded proofs.
ApproxDPNL instead maintains certified lower and upper bounds on the exact output probability throughout inference.
Its stopping condition can therefore be expressed directly in terms of a desired additive or multiplicative error, while the search may be interrupted at any point and still returns a
certified interval.

\paragraph{Surrogate and sampling-based approximation.}
Other approaches avoid exact probabilistic inference by approximating either the inference result or the learning signal.
A-NeSI~\cite{NeSy:A-NeSI:vanKrieken2023} trains neural prediction and explanation models to approximate probabilistic neurosymbolic inference, optionally combined with a symbolic pruner that excludes infeasible partial assignments. ISED~\cite{NeSy:ISED:SolkoBreslin2024} 
instead samples intermediate symbols, evaluates the symbolic program on these samples, and estimates gradients from the resulting partial representation of the program's behavior.
These approaches can avoid the cost of exact inference, but unlike ApproxDPNL they do not maintain certified lower and upper bounds on the exact probability in Eq.~\ref{eq:theory:core}.
In particular, unlike A-NeSI, ApproxDPNL does not require an auxiliary
learned inference model: its approximation guarantees follow directly from the unresolved probability mass during search. 
ApproxDPNL nevertheless retains an explicit search over the assignment
space rather than replacing it with amortized or sampling-based
inference. 
A complementary direction changes the probabilistic modeling assumptions themselves.
Neurosymbolic diffusion models~\cite{vankrieken2025neurosymbolicdiffusionmodels} model dependencies between intermediate symbols through a discrete diffusion process.
DPNL instead retains the independence assumption of the PNL problem considered in this work and focuses on making inference over that model more efficient and controllable.

\paragraph{Hardware-oriented acceleration.}
A complementary line of work improves the execution of neurosymbolic reasoning on modern hardware.
KLAY~\cite{maene2025klayacceleratingarithmeticcircuits} maps arithmetic circuits to a representation designed for efficient parallel evaluation on GPU, while Lobster~\cite{NeSy:Lobster:Biberstein2026} compiles Datalog-based neurosymbolic programs for GPU execution. Dolphin~\cite{NeSy:Dolphin:Naik2025} executes symbolic reasoning on the CPU while vectorizing the probabilistic and gradient computations on the GPU. 
DeepLog~\cite{NeSy:DeepLog:Derkinderen2025} provides a more general neurosymbolic abstract machine based on extended algebraic circuits and GPU execution.
These approaches primarily address the execution of neurosymbolic computations, whereas DPNL changes the inference procedure itself by lazily exploring the assignment space and avoiding complete provenance materialization when possible. The two directions are therefore complementary: hardware-oriented execution techniques could in principle also be investigated for accelerating DPNL.

\section{Conclusion}

We introduced DPNL, an oracle-guided inference framework for
probabilistic neurosymbolic learning that avoids requiring the complete
logical provenance of a query to be materialized before inference.
DPNL instead explores the space of intermediate assignments lazily,
resolving regions as soon as their contribution to the target output
can be determined.
A generic oracle interface decouples this inference procedure from the
representation of the symbolic component.
For valid oracles, DPNL terminates and returns the exact output
probability, while ApproxDPNL can stop the same search early while
maintaining certified lower and upper bounds on that probability,
yielding controlled approximation guarantees.
Our experiments show that DPNL can substantially extend the range of
instances tractable by exact probabilistic inference, while ApproxDPNL
extends this range further through certified approximate inference.
This suggests that probabilistic neurosymbolic inference
can be organized as dynamic, on-demand exploration of the assignment
space, rather than requiring a complete representation of all reasoning
paths before inference begins.
This opens substantial room for further improvements in pruning, region
selection, and branching strategies.

\bibliography{references} 
\bibliographystyle{named}

\newpage
\appendix

\section{Termination and Correctness of DPNL}

We prove that DPNL terminates and returns the exact probability
associated with the target output.

For a valuation \(v\), let
\[
U(v)=\{k\in[1,m]\mid v[k]=\unknown\}
\]
denote the set of unassigned variables.
We define the completion probability of output \(o\) under \(v\) as
\[
p_\theta(o\mid i;v)
:=
\sum_{\substack{
v'\in\mathrm{tot}(v)\\
S(v'[1],\ldots,v'[m])=o
}}
\prod_{k\in U(v)}
\pi_{\theta,k}(v'[k]\mid i).
\]
Thus, the values fixed by \(v\) are held fixed, while the remaining
variables are independently distributed according to their
input-dependent distributions.

These weights define a probability distribution over
\(\mathrm{tot}(v)\), since
\[
\begin{aligned}
\sum_{v'\in\mathrm{tot}(v)}
\prod_{k\in U(v)}
\pi_{\theta,k}(v'[k]\mid i)
&=
\prod_{k\in U(v)}
\sum_{y\in V_k}
\pi_{\theta,k}(y\mid i)\\
&=1.
\end{aligned}
\]

For the empty valuation \(v_{\emptyset}\), this definition coincides
with Eq.~(1):
\[
p_\theta(o\mid i;v_{\emptyset})=p_\theta(o\mid i).
\]

\begin{theorem}
For any input \(i\in I\), output \(o\in O\), valid oracle
\(O_S\) for the symbolic function \(S\), and valuation
\(v\in\mathcal{V}\),
\[
\mathrm{DPNL}(i,o,v,O_S)
=
p_\theta(o\mid i;v).
\]
In particular,
\[
\mathrm{DPNL}(i,o,v_{\emptyset},O_S)
=
p_\theta(o\mid i).
\]
\end{theorem}

\begin{proof}
We proceed by induction on the number of unassigned variables in \(v\):
\[
n(v)=|U(v)|.
\]

\paragraph{Base case.}
Assume that \(n(v)=0\). Then \(v\) is a total valuation and
\(\mathrm{tot}(v)=\{v\}\).

Since \(O_S\) is a valid oracle, it always decides total valuations.
Therefore,
\[
O_S(v,o)=
\begin{cases}
1 & \text{if } S(v[1],\ldots,v[m])=o,\\
0 & \text{otherwise}.
\end{cases}
\]

By definition of the completion probability, and using the convention
that an empty product equals \(1\),
\[
p_\theta(o\mid i;v)=
\begin{cases}
1 & \text{if } S(v[1],\ldots,v[m])=o,\\
0 & \text{otherwise}.
\end{cases}
\]
Hence DPNL returns exactly \(p_\theta(o\mid i;v)\), and terminates.

\paragraph{Induction step.}
Assume the theorem holds for every valuation with at most \(n\)
unassigned variables. Consider a valuation \(v\) with
\(n(v)=n+1\).

There are three possible cases depending on the oracle output.

\paragraph{Case \(O_S(v,o)=1\).}
By validity of the oracle, every total refinement of \(v\) produces
the output \(o\). Hence,
\[
p_\theta(o\mid i;v)
=
\sum_{v'\in\mathrm{tot}(v)}
\prod_{k\in U(v)}
\pi_{\theta,k}(v'[k]\mid i).
\]
Since \(\mathrm{tot}(v)\) contains exactly all combinations of values
of the variables indexed by \(U(v)\), and each
\(\pi_{\theta,k}(\cdot\mid i)\) is a probability distribution,
\[
\begin{aligned}
p_\theta(o\mid i;v)
&=
\prod_{k\in U(v)}
\sum_{y\in V_k}
\pi_{\theta,k}(y\mid i)\\
&=1.
\end{aligned}
\]
DPNL returns \(1\), which is therefore the correct value.

\paragraph{Case \(O_S(v,o)=0\).}
By validity of the oracle, no total refinement of \(v\) produces
the output \(o\). The sum defining the completion probability therefore
contains no satisfying refinement, and
\[
p_\theta(o\mid i;v)=0.
\]
DPNL returns \(0\), which is the correct value.

\paragraph{Case \(O_S(v,o)=\unknown\).}
In this case, DPNL selects an index \(k\) such that
\(v[k]=\unknown\).
The regions represented by the refinements
\(v_{v[k]\leftarrow y}\), for \(y\in V_k\), partition the region
represented by \(v\).

For each \(y\in V_k\),
\[
U(v_{v[k]\leftarrow y})=U(v)\setminus\{k\},
\]
and every
\(v'\in\mathrm{tot}(v_{v[k]\leftarrow y})\)
satisfies \(v'[k]=y\). Hence,
\[
\prod_{j\in U(v)}
\pi_{\theta,j}(v'[j]\mid i)
=
\pi_{\theta,k}(y\mid i)
\prod_{j\in U(v_{v[k]\leftarrow y})}
\pi_{\theta,j}(v'[j]\mid i).
\]

Therefore, partitioning the total refinements of \(v\) according to
the value assigned to \(X_k(i)\) gives
\[
p_\theta(o\mid i;v)
=
\sum_{y\in V_k}
\pi_{\theta,k}(y\mid i)\,
p_\theta(o\mid i;v_{v[k]\leftarrow y}).
\]

For every \(y\in V_k\), the valuation \(v_{v[k]\leftarrow y}\) contains
one fewer unassigned variable:
\[
n(v_{v[k]\leftarrow y})=n(v)-1.
\]
Hence, by the induction hypothesis,
\[
\mathrm{DPNL}(i,o,v_{v[k]\leftarrow y},O_S)
=
p_\theta(o\mid i;v_{v[k]\leftarrow y}),
\]
and each recursive call terminates.

Since
\[
\Pr(X_k(i)=y)=\pi_{\theta,k}(y\mid i),
\]
the recursive computation performed by DPNL returns
\[
\begin{aligned}
&\sum_{y\in V_k}
\Pr(X_k(i)=y)\,
\mathrm{DPNL}(i,o,v_{v[k]\leftarrow y},O_S)
\\
&=
\sum_{y\in V_k}
\pi_{\theta,k}(y\mid i)\,
p_\theta(o\mid i;v_{v[k]\leftarrow y})
\\
&=
p_\theta(o\mid i;v).
\end{aligned}
\]

Since \(V_k\) is finite and every recursive call decreases the number
of unassigned variables, the current call also terminates.

Finally, for the empty valuation \(v_{\emptyset}\),
\[
p_\theta(o\mid i;v_{\emptyset})=p_\theta(o\mid i),
\]
and therefore
\[
\mathrm{DPNL}(i,o,v_{\emptyset},O_S)=p_\theta(o\mid i).
\]
\end{proof}

\section{Logic Oracle Construction}
\label{sec:logic-oracle}

\subsection{Logic-based characterization of symbolic functions}

We consider symbolic functions over finite domains that can be
characterized by logical reasoning.  
For each variable $X_k$ with finite domain $V_k$, we associate each
possible value $x\in V_k$ with a logical formula $\phi_{k,x}$ expressing
that $X_k=x$. 
We assume that, for every $k$, the formulas
\[
\{\phi_{k,x}\mid x\in V_k\}
\]
are mutually exclusive and exhaustive.

For each output value $o\in O$, let $\phi_o$ denote a logical formula
expressing that the symbolic function produces output $o$.

Let $T$ be a logical theory. We assume that the symbolic function $S$
satisfies:
\[
S(x_1,\ldots,x_m)=o
\iff
T\cup
\{\phi_{k,x_k}\mid k\in[1,m]\}
\vdash \phi_o .
\]

In other words, once all variables are assigned, the output of the
symbolic function is determined by logical entailment.

For a valuation $v$, we define the theory induced by $v$ as
\[
T_v =
T\cup
\{\phi_{k,v[k]}\mid v[k]\neq\unknown\}.
\]

Hence, a partial valuation corresponds to a partially instantiated
logical theory.

\subsection{Logic oracle}

The logic oracle uses a decision procedure for the underlying logic to
determine whether the current partial valuation is sufficient to decide
the target output.

\begin{algorithm}
\caption{Logic oracle}
\label{alg:logic-oracle}
\begin{algorithmic}
\REQUIRE valuation $v$, target output $o$
\ENSURE oracle output

\IF{$T_v\vdash\neg\phi_o$}
    \STATE $r\gets0$
\ELSIF{$T_v\vdash\phi_o$}
    \STATE $r\gets1$
\ELSE
    \STATE $r\gets\unknown$
\ENDIF
\STATE \RETURN $r$
\end{algorithmic}
\end{algorithm}

\subsection{Validity and completeness}

\begin{theorem}
Assume that the logical consequence relation is sound.
Then Algorithm~\ref{alg:logic-oracle} defines a valid oracle for $S$.
\end{theorem}

\begin{proof}
Let $v$ be a valuation.

Assume first that the oracle returns $1$ for target output $o$.
Then
\[
T_v\vdash\phi_o .
\]

Let $v'\in\mathrm{tot}(v)$ be any total refinement of $v$.
By definition of refinement, every formula induced by $v$ is also
induced by $v'$, and therefore
\[
T_v\subseteq T_{v'} .
\]

By monotonicity of logical entailment,
\[
T_{v'}\vdash\phi_o .
\]

Using the characterization of $S$, we obtain
\[
S(v'[1],\ldots,v'[m])=o .
\]

Hence every total refinement of $v$ produces output $o$.

Now assume that the oracle returns $0$.
Then
\[
T_v\vdash\neg\phi_o .
\]

For any total refinement $v'\in\mathrm{tot}(v)$,
we again have
\[
T_v\subseteq T_{v'} ,
\]
and therefore
\[
T_{v'}\vdash\neg\phi_o .
\]

Consequently,
\[
S(v'[1],\ldots,v'[m])\neq o .
\]

Hence no total refinement of $v$ produces output $o$.

Therefore the oracle is valid.
\end{proof}

\begin{theorem}
Assume that the logical consequence relation is sound and complete.
Then Algorithm~\ref{alg:logic-oracle} defines a complete oracle for $S$.
\end{theorem}

\begin{proof}
Assume that the oracle returns $\unknown$ for a valuation $v$ and output
$o$.

Then
\[
T_v\nvdash\phi_o
\]
and
\[
T_v\nvdash\neg\phi_o .
\]

By completeness of the logical consequence relation, there exists a
model $I_1$ of $T_v$ such that
\[
I_1\models\phi_o ,
\]
and a model $I_0$ of $T_v$ such that
\[
I_0\models\neg\phi_o .
\]

For $j\in\{0,1\}$, define a valuation $v_j$ by
\[
v_j[k]=x
\quad\text{iff}\quad
I_j\models\phi_{k,x}.
\]

Because $I_j$ satisfies $T_v$, all values fixed by $v$ are preserved.
Therefore,
\[
v_0,v_1\in\mathrm{tot}(v).
\]

Moreover, by the characterization of $S$,
\[
S(v_1[1],\ldots,v_1[m])=o
\]
and
\[
S(v_0[1],\ldots,v_0[m])\neq o .
\]

Thus both a refinement producing $o$ and a refinement not producing $o$
exist in the region represented by $v$.

Therefore returning $\unknown$ is correct, and the oracle is complete.
\end{proof}

\section{Application to DeepProbLog}

\subsection{DeepProbLog as a logic-based symbolic function}

DeepProbLog provides a natural instance of the logic-based symbolic
functions considered in Section~\ref{sec:logic-oracle}.

A DeepProbLog program combines deterministic logic rules with neural
predicates. The outcomes of neural predicates define the intermediate
variables
$X_1,\ldots,X_m$, while the deterministic logic program defines the
symbolic function $S$.

Consider a DeepProbLog query $q$ and let $T$ denote the deterministic
part of the program. For a complete valuation $v$, let $R_v$ denote the
set of neural facts corresponding to the values fixed by $v$.

The symbolic function associated with the query is defined as
\[
S(v[1],\ldots,v[m])=1
\iff
T\cup R_v\vdash q .
\]

Therefore, DeepProbLog fits the logic-based characterization introduced
in Section~\ref{sec:logic-oracle}, where the formulas associated with
the values of neural predicates correspond to the possible neural
facts.

\subsection{DeepProbLog oracle}

For DeepProbLog programs without negation, the logic oracle can be
implemented using the underlying Prolog inference procedure.

For the Boolean query setting of DeepProbLog, we consider
$O=\{0,1\}$ and focus on the target output $o=1$.

For a partial valuation $v$, let $R_v$ denote the neural facts fixed by
$v$. Let $R_v^\star$ denote the optimistic
completion of $R_v$, obtained by assuming all unassigned predicates are
true.

The oracle is defined as follows:

\begin{algorithm}
\caption{DeepProbLog oracle}
\label{alg:deepproblog-oracle}
\begin{algorithmic}
\REQUIRE valuation $v$
\IF{$T\cup R_v\vdash q$}
   \RETURN $1$
\ELSIF{$T\cup R_v^\unknown\nvdash q$}
   \RETURN $0$
\ELSE
   \RETURN $\unknown$
\ENDIF
\end{algorithmic} 
\end{algorithm}

The oracle is valid because DeepProbLog programs without negation are
monotonic: adding additional neural facts cannot invalidate an existing
proof.

The first case certifies that every completion of the current valuation
produces the query, while the second case certifies that no completion
can produce it. Otherwise, further exploration is required.

\subsection{Representing probabilistic choices}

DeepProbLog represents probabilistic choices through mutually exclusive
outcomes, rather than independent Boolean facts.

For a categorical variable with outcomes
$x_1,\ldots,x_n$, directly introducing independent probabilistic facts
would incorrectly allow several outcomes to hold simultaneously.
DeepProbLog therefore uses annotated disjunctions (ADs) and neural
annotated disjunctions (nADs) to represent such choices.

This transformation introduces auxiliary logical variables and rules.
Consequently, the logical representation used for provenance-based
inference may become substantially larger as the number of probabilistic
choices increases.

DPNL does not require this complete logical representation to be
materialized: its oracle-guided exploration can reason directly over
regions of the assignment space.

\subsection{Example: Graph reachability}

Consider a probabilistic directed graph where each potential edge
$(e_i,e_j)$ is represented by a neural predicate
$edge(e_i,e_j)$.

The query asks whether a target node is reachable from a source node.
The deterministic rules can be written as
\begin{align*}
&reachable(e_1).\\
&reachable(Y) \leftarrow reachable(X), edge(X,Y).
\end{align*}
The corresponding symbolic function returns $1$ when the query
$reachable(e_N)$ can be proved.

A provenance-based approach constructs a logical formula representing all
possible paths between the source and target nodes. In a dense graph,
the number of such paths can grow exponentially with the number of
nodes.

In contrast, the DeepProbLog oracle can reason directly over partial
valuations. For example, if a reachable path is already established by
the currently assigned edges, the region can immediately be accepted
without enumerating the remaining edges. Conversely, if reachability is
impossible even when all unknown edges are assumed present, the whole
region can be rejected.

This illustrates how DPNL can exploit logical reasoning directly during
inference instead of first materializing the complete provenance.

\section{Variable Selection Heuristics}

The choice of the next variable to refine determines the structure of
the exploration tree and can substantially affect the number of recursive
calls required by DPNL.

Correctness and termination do not depend on this choice, as every
branching step reduces the number of unassigned variables. However,
problem-specific heuristics may significantly improve pruning
efficiency.

When the oracle is complete, a partial valuation $v$ returning
$\unknown$ indicates that the region represented by $v$ contains both
successful and unsuccessful total refinements. In principle, a complete
oracle may provide information about refinements leading to
each outcome. 
This suggests a branching heuristic that prioritizes variables that are
likely to separate such refinements. More precisely, if the oracle can
provide witness refinements
\[
v^+ , v^- \in \tot{v}
\]
such that
\[
S(v^+[1],\ldots,v^+[m])=o
\]
and
\[
S(v^-[1],\ldots,v^-[m])\neq o ,
\]
one may prioritize variables $X_k$ for which
\[
v^+[k]\neq v^-[k].
\]

Such variables directly distinguish regions that respectively contribute
to or exclude the target output and may lead to smaller exploration
trees.
Designing effective variable-selection heuristics based on oracle
information is an interesting direction for future work.

\section{Experimental Details}
\label{appendix:experimental_details}
\subsection{Hardware}
The symbolic experiments are executed on an Apple MacBook Pro (M1 Max) with 64GB RAM.
The neurosymbolic experiments on Leaf and MNIST-$N$-SUM are executed on a CPU-only Dell Precision 7960 workstation with an Intel Xeon w5-3425 processor (12 cores, 24 threads, 3.2\,GHz base and 4.6\,GHz turbo frequency) and 512GB RAM.
The Visual Sudoku experiments are executed on a server with two Intel Xeon Silver 4114 processors (10 cores and 20 threads each, 2.2\,GHz base and 3.0\,GHz turbo frequency), 187GB RAM and an NVIDIA Quadro P5000 GPU.

\subsection{Implementation Details}

The DPNL framework is implemented in a modular and extensible Python codebase centered around symbolic reasoning with probabilistic inputs. It introduces a unified abstraction for defining symbolic functions, managing structured probabilistic inputs, and executing exact inference using oracle-guided recursive evaluation.
The implementation relies on Pytorch \cite{paszke2019pytorch} and uses Weights and Biases \cite{wandb} as experiment tracking tool. 

\paragraph{Core Abstractions.}
At the heart of the system is the \texttt{Symbolic} class, an abstract interface for symbolic functions $S$, and the \texttt{Input} class, which represents structured inputs composed of random variables (\texttt{RndVar}).
Each \texttt{RndVar} has a discrete domain and a probability distribution, and may be left undefined (\texttt{unknown}) until instantiated during inference.

\paragraph{Oracle Mechanism.}
DPNL oracles implement the \texttt{Oracle} interface, which answers whether $\S{v} = o$ is always true, always false, or unresolved for a given (possibly partial) valuation \(v\) and output \(o\).
Multiple oracle variants are supported:
\begin{itemize}
\item \textbf{Enumeration Oracle:} A naive baseline that returns \texttt{unknown} if any variable in $I$ is undefined. Useful as a reference implementation for exact inference.
\item \textbf{Basic Oracle:} Executes \(\S{v}\) and traps access to undefined variables to detect when the computation depends on an unassigned variable.
\item \textbf{Logic Oracle:} Implementation of the procedure to obtain a logic oracle given an interface with a logic, its solver and the function $S$.
\item \textbf{Combination Oracle:} Sequentially tries multiple oracles, falling back as needed to obtain stronger pruning or efficiency.
\end{itemize}

\paragraph{Recursive Inference.}
The \texttt{PNLProblem} class wraps a symbolic function and input.
It computes the probability that $S(v) = o$ via recursive branching over undefined variables, guided by the oracle.
 
\paragraph{SAT Logic Integration.}
We implemented an interface for the Logic Oracle with the Boolean logic and SAT solvers for the experiments.
The \texttt{SATLogic} class encodes formulas using literals, clauses, and conjunctions, and interfaces with a SAT solver (e.g., MiniSAT).
Symbolic functions defined in SAT are specified through the \texttt{SATLogicSymbolic} interface, which supports variable declarations, clause generation, and relational constraints over domains to easily define $S$.
This enables symbolic functions to be expressed and reasoned about at the propositional level, improving pruning and solver performance.
This experimental instantiation is only one possible implementation of
the generic logic oracle interface.

\subsection{Hyperparameters.} 
\label{appendix:hyperparameters}
The hyperparameters used in the experiments are summarized in Table~\ref{tab:hyperparameters}.
We keep the hyperparameters constant across the experiments with different numbers of digits. 
Further, we conduct 10 independent runs with different random seeds for DPNL. 
We enforce computational timeout: 24h for $N=\{2, 3, 4\} $ and 48h for $N=8$. 
For Visual Sudoku we enforce a computational timeout of 48h. 
\begin{table*}[t]
    \centering
    \footnotesize
    \setlength{\tabcolsep}{4pt}
    \renewcommand{\arraystretch}{1.1}
    \begin{tabular}{@{}llcccccc@{}}
        \toprule
        & \textbf{Engine} & $\mathbf{t}$ \textbf{(s)} & $\boldsymbol{\epsilon}$ & \textbf{Epochs} & \textbf{Batch} & \textbf{Pretrain} & \textbf{Test $t$ (s)} \\
        \midrule
        \multicolumn{8}{@{}l}{\emph{Leaf}} \\
        --                & exact               & --      & --          & $50$ & $4$  & --       & --    \\
        \midrule
        \multicolumn{8}{@{}l}{\emph{MNIST-$N$-SUM (basic oracle)}} \\
        $N \in \{2,3,4\}$ & exact               & --      & --          & $1$  & $2$  & --       & --    \\
        $N=2$             & anytime             & $0.006$ & --          & $1$  & $2$  & --       & $1$   \\
        $N=3$             & anytime             & $0.025$ & --          & $1$  & $2$  & --       & $1$   \\
        $N=4$             & anytime             & $0.25$  & --          & $1$  & $2$  & --       & $1$   \\
        $N=8$             & anytime             & $0.25$  & --          & $10$ & $16$ & --       & $100$ \\
        $N=2$             & $\epsilon$-approx.\ & --      & $0.5$       & $1$  & $2$  & --       & --    \\
        $N \in \{3,4\}$   & $\epsilon$-approx.\ & --      & $0.4$       & $1$  & $2$  & --       & --    \\
        \midrule
        \multicolumn{8}{@{}l}{\emph{Visual Sudoku (Sudoku oracle)}} \\
        $N=4$             & exact               & --      & --          & $10$ & $1$  & --       & --    \\
        $N=4$             & anytime             & $0.08$  & --          & $10$ & $1$  & --       & --    \\
        $N=4$             & $\epsilon$-approx.\ & --      & $10^{-6}$   & $10$ & $1$  & --       & --    \\
        $N=9$             & anytime             & $0.25$  & --          & $25$ & $16$ & $4$ ep.\ & $200$ \\
        \bottomrule
    \end{tabular}
    \caption{Hyperparameters for the neuro-symbolic experiments Leaf, MNIST-N-Sum and Vsudoku.
    All runs use the Adam optimiser with learning rate $1\mathrm{e}{-3}$ ($1\mathrm{e}{-4}$ for Leaf), and CPU-only execution.
    $t$ is the per-sample symbolic timeout for Anytime ApproxDPNL and $\epsilon$ the error budget.
    For Visual Sudoku N9 and MNIST N8, a higher timeout at inference is used. 
    Classifier pretraining uses batch size $32$ and learning rate $1\mathrm{e}{-3}$.
    The total training timeout is $24$--$48$\,h for MNIST-$N$-SUM, $24$\,h (exact) and $5$\,h (anytime, $\epsilon$-approx.) for Visual Sudoku $N=4$, and $48$\,h for $N=9$.}
    \label{tab:hyperparameters}
\end{table*}

\paragraph{MNIST Classifier.}
\label{mnist_classifier}
We follow \cite{NeSy:DeepProbLog:Manhaeve2018} and use the following neural network architecture as digit classifier. 
It consists of two convolutional layers with kernel size 5 and output dimension 6 and 16 (\texttt{nn.Conv2d(1, 6, 5)} and \texttt{nn.Conv2d(6, 16, 5),}) followed by Max pooling and Relu activation (\texttt{nn.MaxPool2d(2, 2)}).
A classifier with three linear layers is stacked which have the dimensions 120 and 84 respectively. 
The output dimension corresponds to the number of digits and equals 10. 

\subsection{Baselines} To obtain the results with the baselines for the MNIST-N-Sum task, we used the source code from the following publicly available code repositories: 
\begin{itemize}
    \item DeepProbLog \& DPLA$^*$: 
    \item []\url{https://github.com/ML-KULeuven/deepproblog}
    \item A-Nesi: \url{https://github.com/HEmile/a-nesi}
    \item Scallop: \url{https://github.com/scallop-lang/scallop}
\end{itemize}
For the Leaf task, all baselines were reproduced with the code from the ISED repository: \url{https://github.com/alaiasolkobreslin/ISED}.
\paragraph{A-Nesi.}
We conduct five independent runs with different random seeds for each experiment. 
Table~\ref{tab:hyperparameters:anesi} shows the hyperparameters used for configurations used for the different variants of A-Nesi \cite{NeSy:A-NeSI:vanKrieken2023}. 
\begin{table}[H]
    \small
    \centering
    \begin{tabular}{|c|c|c|c|}
    \hline
    \textbf{Key} & \textbf{Explain} & \textbf{Predict Only} & \textbf{Prune} \\ \hline
    \textbf{amt\_samples} & 600 & 600 & 600 \\ \hline
    \textbf{K\_beliefs} & 2500 & 2500 & 2500 \\ \hline
    \textbf{nrm\_lr} & 0.001 & 0.001 & 0.001 \\ \hline
    \textbf{perception\_lr} & 0.001 & 0.001 & 0.001 \\ \hline
    \textbf{dirichlet\_lr} & 0.01 & 0.01 & 0.01 \\ \hline
    \textbf{dirichlet\_iters} & 50 & 50 & 50 \\ \hline
    \textbf{dirichlet\_init} & 0.1 & 0.1 & 0.1 \\ \hline
    \textbf{dirichlet\_L2} & 900000 & 900000 & 900000 \\ \hline
    \textbf{nrm\_loss} & mse & mse & mse \\ \hline
    \textbf{hidden\_size} & 800 & 800 & 800 \\ \hline
    \textbf{layers} & 3 & 3 & 3 \\ \hline
    \textbf{prune} & False & False & True \\ \hline
    \textbf{predict\_only} & False & True & False \\ \hline
    \textbf{use\_prior} & True & True & True \\ \hline
    \end{tabular}
    \caption{Table of hyperparameters of A-Nesi for the variants explain, predict only, and prune.\label{tab:hyperparameters:anesi}}
\end{table}

\paragraph{Scallop.}
Since Scallop has been only applied to the MNIST-1-SUM task \cite{NeSy:Scallop:Huang2021}, we extend the experiments with Scallop to the MNIST-N-SUM task for N=\{2,3,4\} in order to use it as baseline. 
The logic programs are given in Listing~\ref{lst:scallop_rules}.
We use Scallop with exact as well as with approximate reasoning. 
For approximate reasoning of Scallop we choose the \texttt{diff-top-k-proof} semiring with $k=3$, as recommended in \cite{NeSy:Scallop:Huang2021}. 
To investigate the performance of Scallop with exact reasoning, we choose $k$ in a way that the logic provenance formula is in any case smalle than $k$ and is consequently not pruned. 
In particular, we choose $k$ for every MNIST-N-SUM task depending on N: 
$$k = 2^{2 \cdot 10 \cdot N}.$$ 
This way, $k$ represents an upper bound for the length of the logical provenance formula. 

\paragraph{DPLA$^*$. } 
DPLA$^*$ employs heuristics to estimate the probability of partial proofs to search for the best proof. 
Here, we use the geometric mean heuristic \cite{NeSy:DeepProbLog:Manhaeve2021}.

\begin{lstlisting}[caption={Rules used in Scallop for the MNIST-N-SUM task}, label={lst:scallop_rules}]
# N=1
self.scl_ctx.add_rule("summand_one(a) :- digit_1(a)")
self.scl_ctx.add_rule("summand_two(a) :- digit_2(a)")
# N=2
self.scl_ctx.add_rule("summand_one(10 * a + b) :- digit_1(a), digit_2(b)")
self.scl_ctx.add_rule("summand_two(10 * a + b) :- digit_3(a), digit_4(b)")
# N=3
self.scl_ctx.add_rule("summand_one(10 * 10 * a + 10 * b + c) :- digit_1(a), digit_2(b), digit_3(c)")
self.scl_ctx.add_rule("summand_two(10 * 10 * a + 10 * b + c) :- digit_4(a), digit_5(b), digit_6(c)")
# N=4
self.scl_ctx.add_rule("summand_one(10 * 10 * 10 * a + 10 * 10 * b + 10 * c + d) :- digit_1(a), digit_2(b), digit_3(c), digit_4(d)")
self.scl_ctx.add_rule("summand_two(10 * 10 * 10 * a + 10 * 10 * b + 10 * c + d) :- digit_5(a), digit_6(b), digit_7(c), digit_8(d)")
\end{lstlisting}

\end{document}